\documentclass{article}

     \PassOptionsToPackage{numbers, compress}{natbib}
     \usepackage[final]{neurips_2020}

\usepackage[utf8]{inputenc} % allow utf-8 input
\usepackage[T1]{fontenc}    % use 8-bit T1 fonts
\usepackage{hyperref}       % hyperlinks
\usepackage{url}            % simple URL typesetting
\usepackage{booktabs}       % professional-quality tables
\usepackage{amsfonts}       % blackboard math symbols
\usepackage{nicefrac}       % compact symbols for 1/2, etc.
\usepackage{microtype}      % microtypography
\usepackage{graphicx}
\usepackage{caption}
\usepackage{multirow}
\usepackage{algorithm}
\usepackage{algorithmic}
\usepackage{wrapfig}
\usepackage{listings}
\usepackage[nounderscore]{syntax}
\usepackage[inline]{enumitem}
\usepackage{Definitions}
\usepackage[dvipsnames]{xcolor}

\hypersetup{
colorlinks=true,
}

\newcommand{\maybetodo}[1]{{#1}}

\newcommand{\charles}[1]{\maybetodo{\color{Plum} [Charles: #1]}}

\newcommand{\modelshort}{ALOE}
\newcommand{\defeq}{:=}
\newcommand{\M}[1]{\text{\tiny #1}}

\title{Learning Discrete Energy-based Models via \\Auxiliary-variable Local Exploration}

\author{%
  Hanjun Dai, Rishabh Singh, Bo Dai, Charles Sutton, Dale Schuurmans\\
  Google Research, Brain Team\\
  \texttt{\{hadai, rising, bodai, charlessutton, schuurmans\}@google.com}\\
}

\begin{document}

\maketitle
\vspace{-3mm}
\begin{abstract}
Discrete structures play an important role in applications
like program language modeling and software engineering.
Current approaches to predicting complex structures typically consider
autoregressive models for their tractability,
with some sacrifice in flexibility.  
Energy-based models (EBMs) on the other hand offer a more flexible
and thus more powerful approach to modeling such distributions,
but require partition function estimation.
%This challenge is exacerbated by the absence of gradients over discrete data.
In this paper we propose \modelshort, a new algorithm for learning conditional
and unconditional EBMs for discrete structured data,
where parameter gradients are estimated using a learned sampler
that mimics local search.
We show that the energy function and sampler can be trained efficiently
via a new variational form of power iteration,
achieving a better trade-off between flexibility and tractability. 
Experimentally, we show that learning local search leads to significant
improvements in challenging application domains.
%program synthesis and software engineering.
Most notably, we present an energy model guided fuzzer for software testing
that achieves comparable performance 
to well engineered fuzzing engines like libfuzzer. % on some targets.
\end{abstract}

\setlength{\abovedisplayskip}{2pt}
\setlength{\abovedisplayshortskip}{2pt}
\setlength{\belowdisplayskip}{2pt}
\setlength{\belowdisplayshortskip}{2pt}
\setlength{\jot}{2pt}

\setlength{\floatsep}{2ex}
\setlength{\textfloatsep}{2ex}

%\bodai{@Hanjun, 
%\begin{itemize}
%  \item please use `energy function' instead of 'score function';
%  \item please check to avoid bias reduction;
%\end{itemize}
%}

% !TEX root = main.tex
\section{Introduction}

%\bodai{I think we need more discussion about the motivation for discrete structured data modeling. } \charles{I tried to make the sentence a bit more specific.}
Many real-world applications involve prediction of discrete structured
data,
such as syntax trees for natural language processing~\citep{tai2015improved,Zhang2016-nk},
sequences of source code tokens for program synthesis~\citep{devlin2017robustfill},
and structured test inputs for software testing~\citep{godefroid2017learn}.
A common approach for modeling a distribution over structured data
is the autoregressive model.
Although any distribution can be factorized in such a way,
the parameter sharing used in neural autoregressive models 
can restrict their flexibility.
Intuitively, a standard way to perform inference with autoregressive models has a single pass with a
predetermined order,
which forces commitment to early decisions that cannot
subsequently be rectified.
Energy-based models~\citep{lecun2006tutorial} (EBMs), on the other hand,
define the distribution with an \emph{unnormalized} energy function,
which allows greater flexibility by not committing to \emph{any}
inference order.
In principle,
this allows more flexible model parameterizations such as bi-directional LSTMs,
tree LSTMs~\citep{tai2015improved,Zhang2016-nk},
and graph neural networks~\citep{scarselli2008graph, kipf2016semi}
to be used to capture non-local dependencies.

Unfortunately, the flexibility of EBMs 
exacerbates the difficulties of learning and inference,
since the partition function is typically intractable.
EBM learning algorithms therefore employ approximate strategies
such as contrastive learning,
where positive samples are drawn from data
and negative samples obtained from an alternative sampler
\citep{dai2019exponential}.
Contrastive divergence~\citep{hinton2002training,tieleman2008training,wu2018sparse,du2019implicit},
pseudo-likelihood~\citep{besag1975statistical} and
score matching~\citep{hyvarinen2005estimation}
are all examples of such a strategy.
However, 
such approaches use hand-designed negative samplers,
which can be overly restrictive in practice,
thus~\citep{dai2019exponential,dai2018kernel,arbel2020kale}
consider joint training of a flexible negative sampler 
along
with the energy function,
achieving significant improvements in model quality.
These recent techniques are not directly applicable to discrete structured data
however, since they exploit gradients over the data space.
In addition,
the parameter gradient involves an intractable sum,
which also poses a well-known challenge for stochastic estimation~\citep{williams1992simple,glynn1990likelihood,bengio2013estimating,jang2016categorical,maddison2016concrete,tucker2017rebar,tucker2018mirage,yin2019arsm}.

% require the gradient with respect to the sample, which does not directly apply to discrete structured data. Moreover, one needs to use REINFORCE gradient estimator~\citep{williams1992simple,glynn1990likelihood}, which introduces high variance in gradient estimators, and thus, slow convergence in stochastic optimization.
%\charles{I tried to clarify that the gradients in (i) and (ii) are different, otherwise it sounds like we are contradicting ourselves "the gradient is not defined, and it's well known that it's difficult to estimate" :-)}\bodai{This new claim is more clear and succinct, and it is exact what I would like to convey. }

In this work, we propose 
\emph{Auxiliary-variable LOcal Exploration (\modelshort)},
a new method for discrete EBM training with a learned negative sampler.
Inspired by viewing MCMC as a local search in continuous space,
% the Hamiltonian Monte Carlo~(HMC~\citep{neal2011mcmc}) and stochastic gradient Langevin dynamics~(SGLD~\citep{welling2011bayesian}) as local search algorithms in continuous domains with gradient guidance, 
we parameterize the learned sampler using local discrete search;
that is, the sampler first generates an initial negative structure using
a tractable model, such as an autoregressive model,
then repeatedly makes local changes to the structure.
This provides a learnable negative sampler that still depends globally
on the sequence.  
As there are no demonstrations for intermediate steps in the local search,
we treat it as an auxiliary variable model.
To learn this negative sampler, instead of the %vanilla 
primal-dual form of MLE \cite{Wainwright2008-uc,dai2019exponential}, 
%\charles{@Bo, Hanjun: are the previous cites OK?}\bodai{Sure, this is more comprehensive.}
we propose a new variational objective that uses \emph{finite-step}
MCMC sampling for the gradient estimator,
resulting in an efficient method. % for the gradient estimation problem.
The procedure alternates between updating the energy function
and improving the dual sampler by power iteration, which
can be understood as generalization of persistent
contrastive divergence (PCD~\citep{tieleman2008training}).

% This procedure is an  importance reweighted gradient estimator that alternates between 1) learning the sampler with MLE; 2) proposing biased samples with minimum number of Gibbs sampling steps that initialized from sampler.\charles{Item (2) is a bit difficult to parse.} This principle can also be understood as generalization of persistent contrastive divergence (PCD~\citep{tieleman2008training}) that distills MCMC samples into the sampler. 

%In summary, our contributions in this paper are:
%\begin{itemize}[leftmargin=*,nolistsep,nosep]
%	\item We propose \modelshort, a new discrete EBM learning method, which learns a local search algorithm with auxiliary variables for negative sampling in discrete space. The sampler is learned by variational power iteration, which avoids the need of policy gradient. 
%	\item We introduce practical proposals for importance reweighted gradient estimator in learning the auxiliary-variable local search \emph{without intermediate supervision}, based on inverse sampling or shortest edit distance.
%	\item We experimentally evaluate the approach on both synthetic and real-world tasks. For a program synthesis problem, we observe significant accuracy improvements over the baseline methods. More notably, for a software testing task, a fuzz test guided by an EBM achieves comparable performance to a well-engineered fuzzing engine on several open source software projects.
%\end{itemize}

We experimentally evaluated the approach on both synthetic and real-world tasks. For a program synthesis problem, we observe significant accuracy improvements over the baseline methods. More notably, for a software testing task, a fuzz test guided by an EBM achieves comparable performance to a well-engineered fuzzing engine on several open source software projects.

% !TEX root = main.tex
\vspace{-1mm}
\section{Preliminaries}
\label{sec:model}
\vspace{-1mm}
%\paragraph{Energy-based Models} 
{\bf Energy-based Models: } 
Let $x \in \mathcal{S}$ be a discrete structured datum in the space
$\mathcal{S}$.
We are interested in learning an energy function
$f: \mathcal{S} \rightarrow \RR$
that characterizes the distribution on $\mathcal{S}$.
Depending on the space, % $S$, 
$f$ can be realized as an LSTM~\citep{hochreiter1997long}
for sequence data, a tree LSTM~\citep{tai2015improved}
for tree structures,
or a graph neural network~\citep{scarselli2008graph} for graphs.
The probability density function is defined as  
\begin{equation}
    p_f(x) = \exp\rbr{f(x) - \log Z_f} \propto \exp\rbr{f(x)},
\end{equation}
where $Z_f \defeq \sum_{x \in \mathcal{S}} \exp\rbr{f(x)}$ is the partition function.

% The above equation defines an unconditional energy model.  
It is natural to extend the above model for conditional distributions.
Let $z \in \mathcal{Z}$ be an arbitrary datum in the space $\mathcal{Z}$.
Then a conditional model is given by the density
\begin{equation}
	p_f(x|z) = \frac{\exp\rbr{f(x, z)}}{Z_{f, z}}, \text{ where } Z_{f, z} = \sum_{x \in \mathcal{S}} \exp\rbr{f(x, z)}.
\end{equation}
Typically $\mathcal{S}$ is a combinatorial set, which makes the partition
function $Z_f$ or $Z_{f, z}$ intractable to calculate.
This makes both learning and inference difficult. 

%\vspace{-2mm}
%\paragraph{Primal-Dual view of MLE} 
{\bf Primal-Dual view of MLE: } 
Let $\Dcal = \cbr{x_i}_{i=1}^{|\Dcal|}$ be a sample obtained from some unknown
distribution over $\mathcal{S}$.
%Our learning objective is to maximize 
We consider maximizing
the log likelihood of $\Dcal$ under model $p_f$:
\begin{equation}\label{eq:mle}
	\max_f \,\, \ell\rbr{f}\defeq \EE_{x \sim \Dcal}\sbr{f(x)} - \log Z_{f}.
\end{equation}
Directly maximizing this objective is not feasible due to the intractable
log partition term.
Previous work \citep{dai2018kernel,dai2019exponential} reformulates the MLE
by exploiting the Fenchel duality of the $\log$-partition function,
\ie, $\log Z_f = \max_q \EE_{x \sim q}\sbr{f(x)} - H(q)$, where $H(q) = -\EE_q \sbr{\log q}$ is the entropy of $q\rbr{\cdot}$, which leads to a primal-dual view of the MLE:
\begin{equation}
    \max_f \min_q\,\, \bar\ell\rbr{f, q}\defeq \underbrace{\EE_{x \sim \Dcal}\sbr{f(x)}}_{\text{positive sampling}} - \underbrace{\EE_{x \sim q}\sbr{f(x)}}_{\text{negative sampling}} - H(q)
    \label{eq:minimax}
\end{equation}
%As one can see, 
Although the primal-dual view introduces an extra dual
distribution %(or sampler) 
$q\rbr{x}$ for negative sampling,
%which is to be learned, it also provides the 
this provides an
opportunity to %exploit an
use a
trainable
deep neural network %parametrization 
to capture the intrinsic data manifold,
which can lead
%leading 
to a better negative sampler.
In \citep{dai2019exponential},
a family of flexible %parametrizations for 
negative samplers
was introduced,
which combines learnable components with
dynamics-based MCMC samplers, \eg, Hamiltonian Monte Carlo (HMC)
\cite{neal2011mcmc} and
stochastic gradient Langevin dynamics (SGLD) \cite{welling2011bayesian},
to obtain significant practical improvements in continuous data modeling. 
%
%The success of ADE 
However, the success of this approach
relied on the differentiability of $q$ and $f$ over a continuous domain,
%$x \in \RR^d$,
requiring guidance not only from $\nabla_x f\rbr{x}$, but also from
gradient back-propagation through samples, \ie, $\nabla_\phi \bar\ell\rbr{f, q} = -\nabla_\phi\EE_{x\sim q_\phi}\sbr{\nabla_x f\rbr{x}\nabla_\phi x}$ where $\phi$ denotes the parameters of the dual distribution.
Unfortunately, for discrete data,
learning a
%both the 
dual distribution for negative sampling %and its learning are 
is difficult.
Therefore  this approach is not directly translatable to discrete EBMs. 
% Using Fenchel dual of log-partition function~, we can obtain the lowerbound of it by optimizing another variational distribution (\ie, the negative sampler):  ADE~\citep{dai2019exponential} thus uses the primal-dual learning objective: \charles{It would probably be helpful to the reader to explain the term negative sampler if we have space.}
% Denote $L(q) := -\EE_{x \sim q}\sbr{f(x)} - H(q)$. 
%  When $x$ belongs to some discrete structure space $\mathcal{S}$, it is nontrivial to optimize. 

% =======
% Typically $\mathcal{S}$ is combinatorially large, which makes the partition function $Z_f$ or $Z_{f, z}$ computationally intractable to calculate. This brings difficulty in both learning and inference of such models. Recent advances~\citep{du2019implicit,dai2019exponential,arbel2020kale} in learning with energy models mostly focused on energy functions defined on continuous data, and techniques like HMC and SGLD which rely on gradients cannot be generalized to discrete domain. Inspired by these successes on continuous data, we thus propose a new learning algorithm for discrete structures inheriting the benefits while eliminating the difficulties.

% !TEX root = main.tex

\vspace{-1mm}
\section{Auxiliary-variable Local Exploration}
\label{sec:learning}
\vspace{-1mm}
To extend the above approach to discrete domains,
we first introduce a variational form of power iteration
(Section~\ref{sec:variational_pi})
combined with local search (Section~\ref{sec:learn_local_search}).
We present the method for an unconditional EBM,
but the extension to a conditional EBM is straightforward.

%We propose a new learning algorithm for discrete structures to keep the benefits of learning sampler like in ADE while circumventing the difficulties on discrete domain. We first introduce the learning mechanism based on variational form of power iteration to bypass the gradient back-propagation through discrete samples in~Section~\ref{sec:variational_pi}. The local search negative sampler is then introduced in Section~\ref{sec:learn_local_search}. Without loss of generality, here we describe our proposed learning algorithm for unconditional EBMs for discrete structures. The learning algorithm of the conditional models can be derived in a similar way.

% !TEX root = main.tex

\subsection{MLE via Variational Gradient Approximation}
\label{sec:variational_pi}

For discrete data, learning the dual sampler in the min-max form of 
MLE~\eqref{eq:minimax} is notoriously difficult, usually leading to inefficient gradient 
estimation~\citep{williams1992simple,glynn1990likelihood,bengio2013estimating,jang2016categorical,maddison2016concrete,tucker2017rebar,tucker2018mirage,yin2019arsm}.
Instead
we consider an alternative optimization 
that has the same solution but is computationally preferable:
\begin{align}\label{eq:stat_constraint}
\max_{f, q} \,\,  \tilde\ell\rbr{f, q}\defeq \max_f \max_{q \in \Kcal} \: {\EE_{x\sim\Dcal}\sbr{f\rbr{x}} - \EE_{x\sim q}\sbr{f\rbr{x}}}, \\
\Kcal\defeq \cbr{q \:\middle|\: \int q\rbr{x}k_f\rbr{x'|x}dx  = q\rbr{x'}, \forall x'\in \Scal},
\end{align}
where $k_f\rbr{x'|x}$ is any ergodic MCMC kernel whose stationary distribution is $p_f$. 
\begin{theorem}\label{thm:equivalent}
Let $\rbr{f^*, q^*} = \argmax_{f, q} \tilde\ell\rbr{f, q}$.
If the kernel $k_f\rbr{x'|x}$ is ergodic with stationary distribution $p_f$,
then $f^* = \argmax \ell\rbr{f}$ is the MLE and $q^* = p_{f^*}$. 
\end{theorem}
\begin{proof}
% By the first-order optimality condition, we have that
% the optimal $f^*$ should satisfy
% \begin{eqnarray}\label{eq:kkt_f}
% \EE_{x\sim\Dcal}\sbr{\nabla_f f\rbr{x}} - \EE_{x\sim q}\sbr{\nabla_f f\rbr{x}} = 0. 
% \end{eqnarray}
By the ergodicity of $k_f\rbr{x'|x}$, there is unique feasible solution satisifying the constraint $\int q\rbr{x}k_f\rbr{x'|x}dx  = q\rbr{x'}$, which is $p_f\rbr{x}$.
%Plugging this into~\eqref{eq:kkt_f} yields
Substituting this into the gradient of $\tilde{\ell}$ yields
$$
\EE_{x\sim\Dcal}\sbr{\nabla_f f\rbr{x}} - \EE_{x\sim q_f}\sbr{\nabla_f f\rbr{x}} = 0,
$$
verifying that $f$ is the optimizer of~\eqref{eq:mle}. 
% Therefore, the optimal $q$ does not change if more regularization for $q$ is added to the objective. We consider the regularization $H\rbr{q}$, which leads to
% \begin{eqnarray*}
% \max_{f, q}\,\, \EE_{x\sim\Dcal}\sbr{f\rbr{x}} - \EE_{x\sim q}\sbr{f\rbr{x}} - H\rbr{q}
% % \quad  \st  \int q\rbr{x}k_f\rbr{x'|x}dx  = q\rbr{x'}.
% \end{eqnarray*}
\vspace{-2mm}
\end{proof} 

% To further make the optimization computatble, we replace the pointwise constraints with 
% $$
% q\rbr{x} = \min_q D_{KL}\rbr{\int q\rbr{x}k_f\rbr{x'|x}dx  || q\rbr{x'}}, 
% % \propto \EE_{ q\rbr{x}k_f\rbr{x'|x} }\sbr{q\rbr{x}},
% $$
% leading to our final optimization,
% \begin{eqnarray}\label{eq:stat_constraint}
% \max_{f} \,\,  \dot\ell\rbr{f, q}\defeq \cbr{\EE_{x\sim\Dcal}\sbr{f\rbr{x}} - \EE_{x\sim q}\sbr{f\rbr{x}}\bigg|\quad \st\,\, q = \argmin_q D_{KL}\rbr{\int q\rbr{x}k_f\rbr{x'|x}dx  || q\rbr{x'}}}.
% \end{eqnarray}
% which is a bi-level optimization problem~\citep{anandalingam1992hierarchical,colson2007overview}. 
% We consider the $\min_q D_{KL}\rbr{\int q\rbr{x}k_f\rbr{x'|x}dx  || q\rbr{x'}}$ purely from the computational perspective 

Solving the optimization~\eqref{eq:stat_constraint} is still nontrivial,
as the constraints are in the function space.
We therefore propose an alternating update 
based on the variational form~\eqref{eq:stat_constraint}:
\begin{itemize}[leftmargin=*,nolistsep,nosep]
	\item {\bf Update $q$ by power iteration:}
Noticing that the constraint actually seeks an eigenfunction of
$k_f\rbr{x'|x}$, we can apply power iteration to find the optimal $q$. 
Conceptually, this power iteration executes
$q_{t+1}(x') = \int q_t(x) k_f\rbr{x'|x}dx$ until convergence.
However, since the integral is intractable,
we instead apply a variational formulation to minimize 
	\begin{equation}
	q_{t+1} = \argmin_q D_{KL}\rbr{\int q_t(x) k_f\rbr{x'|x}dx\Big|\Big| q} = \argmin_q \EE_{q_t\rbr{x}k_f\rbr{x'|x}}\sbr{\log q\rbr{x'}}. 	\label{eq:variational_pi}
	\end{equation}
In practice, this only requires a few power iteration steps.
%which saves computational cost.
Also we do not need to worry about differentiability with respect to $x$,
as \eqref{eq:variational_pi} needs to be differentiated only with respect to
the parameters of $q.$ 
	%\charles{@Bo: Does the previous sentence still say what you intended?}\bodai{reply: I guess I was trying to emphasize we can bypass the difficulty in computing gradient w.r.t. sampler parameters in previous work with $KL$-div. }
We will show in the next section that this framework actually allows a much
more flexible $q$ than autoregressive, such as a local search algorithm. 

	% \bodai{should we say some about practical number of iteration?}

	\item {\bf Update $f$ with MLE gradient:}
Denote $q_f^* = \argmax_{q\in \Kcal}\tilde\ell\rbr{f, q}$.
Then $q^*_f$ converges to $p_f$. 
Recall the unbiased gradient estimator for MLE $\ell\rbr{f}$ w.r.t. $f$ is
	\begin{eqnarray*}
	\nabla_f \ell\rbr{f} = \EE_{x\sim\Dcal}\sbr{\nabla_f f\rbr{x}} - \EE_{x\sim q^*_f}\sbr{\nabla_f f\rbr{x} },
	\end{eqnarray*}
	% where the $q_T$ is the output of the power iteration as $q_T = \argmax_{q\in \Kcal}\tilde\ell\rbr{f, q}$. 
	% The gradient computation comes from Danskin's theorem. 
\end{itemize}
By alternating these two updates,
we obtain the ALOE framework illustrated in Algorithm~\ref{alg:aloe}. 

\noindent\textbf{Connection to PCD:}
When we set the number of power iteration steps to be $1$,
the variational form of MLE optimization can be understood
as a function generalized version of
Persistent Contrastive Divergence (PCD)~\citep{tieleman2008training},
where we distill the past MCMC samples into the sampler
$q$~\citep{xie2018cooperative}. Intuitively, since $f$ is optimized by
gradient descent, the energy models between adjacent stochastic gradient
iterations should still be close, and the power iteration will converge
very fast. 
%
% %%% sacrificed to the gods of space limits
%
%  Moreover, our method has additional benefits compared to PCD:
%\begin{itemize}[leftmargin=*,nolistsep,nosep]
%	\item Distilling $p_f$ using $q$ can be generalized to conditional energy models. Sampling from $p_f(x|y)$ can be implemented by parameterizing the sampler as $x \sim q_y$. 
%	\item The samples between different MCMC chains are independent, since we initialize the chain from $q$ rather than inheriting directly from the previous chains.   
%\end{itemize}

\noindent\textbf{Connection to wake-sleep algorithm:}
ALOE is also closely related to the ``wake-sleep''
 algorithm~\citep{hinton1995wake}
introduced for learning Helmholtz machines~\citep{dayan1995helmholtz}.
The ``sleep'' phase learns the recognition network with objective
$D_{KL}(p_f || q)$, requiring samples from the current model.
However it is hard to obtain such samples for general EBMs,
so we exploit power iteration in a variational form.
% to promote the negative sampler.

%\charles{This paragraph is true for $D_{KL}\rbr{q || p_f}$ as well; it is not specific to our method.}

% !TEX root = main.tex

\subsection{Negative sampler as local search with auxiliary variables}
\label{sec:learn_local_search}

\begin{figure*}[t]
\begin{minipage}{0.49\textwidth}
\begin{algorithm}[H]
    \centering
    \caption{Main algorithm of ALOE}\label{alg:aloe}
    \begin{algorithmic}[1]
    \STATE Input: Observations $\Dcal = \cbr{x_i}_{i=1}^{|\Dcal|}$
    \STATE Initialize score function $f$, sampler $q$.
    \FOR{$x \sim \Dcal$}
    \STATE Sample $(\hat{x}, \tilde{x})$ from $q(\hat{x})k_f(\tilde{x}|\hat{x})$
    \STATE Update $f$ with $-\nabla_f f(x) +\nabla_f f(\tilde{x})$
    \STATE Update $q$ using \algref{alg:learn_q}
    \ENDFOR
    \end{algorithmic}
\end{algorithm}
\end{minipage}
\hfill
\begin{minipage}{0.49\textwidth}
\begin{algorithm}[H]
    \centering
    \caption{Update sampler $q$}\label{alg:learn_q}
    \begin{algorithmic}[1]
	\STATE Input: Current model $f$
	\FOR{$i\leftarrow 1 $ to \# power iteration steps}
	\STATE Sample $\tilde{x}$ from $q$, and get $x$ from $k_f(\cdot|\tilde{x})$.
	\STATE Sample trajectories $\textstyle \{\xb^j_{0:t^j}\}_{j=1}^N$ for $x$ using Eq~\eqref{eq:inv_proposal} or Eq~\eqref{eq:short_proposal}.
	\STATE Update $q$ with gradient from Eq~\eqref{eq:monte_carlo_grad}.
	\ENDFOR
    \end{algorithmic}
\end{algorithm}
\end{minipage}

\caption{
ALOE for learning unconditional discrete EBMs.
Algorithms are similar for conditional case.
We demonstrate with a single example,
but in practice batched optimization is used.
\label{fig:algo}
}
\end{figure*}

%Minimizing $D_{KL}(p_f || q)$ requires a flexible enough $q$ to minimize the bias introduced during optimization. 
Ideally the sampler $q*$ should converge to the stationary distribution $p_f$,
which requires a sufficiently flexible distribution. % to capture the EBM. 
One possible choice for a discrete structure sampler is an autoregressive
model, like RobustFill for generating program trees~\citep{devlin2017robustfill}
or GGNN for generating graphs~\citep{li2018learning}.
However, these have limited flexibility due to 
%However, the limitation of the capacity comes from the fact that 
parameters being shared at each decision step,
which is
needed to handle variable sized structures.
%Such requirement is due to the need of handling variable sized structures. 
Also the ``one-pass'' inference according to a predefined order makes the initial decisions too important in the entire sampling procedure. 

Intuitively, humans do not generate structures sequentially,
but perform successive refinement.
%write the code in a linear way. Instead, we typically perform several refinement steps like refactoring and potential bug correction, on top of the initial scratch code. 
Recent approaches 
for continuous EBMs
have found that using HMC or SGLD provides more effective learning
~\citep{dai2019exponential, du2019implicit}
by exploiting gradient information.
%These inference strategies iteratively update a current data point
%The HMC or SGLD iteratively updates the current samples 
%with the guidance of gradient information. Here 
For discrete data, 
an analogy to gradient based search is local search.
In discrete local search,
an initial solution can be obtained using a simple algorithm,
then local modification can be made to successively improve the structure.
%while in following stages it modifies the current solution to a new one with potentially better quality.

By parameterizing $q$ as a local search algorithm,
we obtain a strictly more flexible sampler than the autoregressive counterpart.
Specifically, we first generate an initial sample $x_0 \sim q_0$,
where $q_0$ can be an autoregressive distribution with parameter sharing,
or even a fully factorized distribution.
Next we obtain a new sample using an editor $q_A(x_i | x_{i-1})$,
where $q_A(\cdot|\cdot): \mathcal{S} \times \mathcal{S} \mapsto \RR$
defines a transition probability.
We also maintain a stop policy $q_{\M{stop}}(\cdot): \mathcal{S} \mapsto [0, 1]$
that decides when to stop %such modification. 
editing.
The overall local search procedure yields a chain of $\xb_{0:t} := \cbr{x_0, x_1, \ldots, x_t}$, with probability
\begin{equation}
	q(\xb_{0:t}; \phi) = q_0(x_0) \prod_{i=1}^t q_A(x_i|x_{i-1}) \prod_{i=0}^{t-1} (1-q_{\M{stop}}(x_{i})) q_{\M{stop}}(x_t)
	\label{eq:traj_q}
\end{equation}
where $\phi$ denotes the parameters in $q_0, q_A$ and $q_{\M{stop}}$. 
The marginal probability of a sample $x$ is:
\begin{equation}
	q(x;\phi) = \sum_{t, \xb_{0:t}: t \leq T} q(\xb_{0:t}; \phi) \II\sbr{x_t = x}, \text{ where $T$ is a maximum length},
	\label{eq:prob_q}
\end{equation}
which we then use as the variational distribution in Eq~\eqref{eq:variational_pi}.
The variational distribution $q$ can be viewed as a latent-variable model, where  $x_0, \ldots, x_{t-1}$ are the latent variables.
%where $q$ maximizes the probability of samples $x$ obtained from $q'(x')k(x|x')$, we can only observe $x_t=x$ but not $x_0, \ldots, x_{t-1}$. This implies that Eq~\eqref{eq:prob_q} is essentially a latent variable model (LVM). 
This choice is expressive, but it brings the difficulty of optimizing
\eqref{eq:variational_pi} due to the intractability of marginalization.
Fortunately, we have the following theorem for an unbiased gradient estimator:
\begin{theorem}
\vspace{-3mm}
 \citet{steinhardt2015learning}: the gradient with respect to parameters $\phi$ has the form
\begin{equation}
	\nabla_{\phi} \log q(x;\phi) = \EE_{q(\xb_{0:t} | x_t = x; \phi)} \sbr{ \nabla_{\phi} \log q([\xb_{0:t-1}, x];\phi) }  
	\label{eq:grad_q}
\end{equation}
where $q(\xb_{0:t} | x_t = x; \phi) \propto q(\xb_{0:t}; \phi) \II\sbr{x_t = x}$. 
\label{thm:grad_q}
\vspace{-3mm}
\end{theorem}
%\charles{I would prefer to minimize the number of new letters here. Do we need to call this $g_x$? Could we not call it $q_\phi(x_{0:t-1}),$ as this is a joint distribution? Do we need to call the conditional distributions $A$ instead of $q$? The reason I am suggesting that we call everything
%by the same letter is to make clear that this is all part
%of the same joint distribution.}
In above equation,
$q(\xb_{0:t} | x_t = x; \phi)$ 
is the posterior distribution given the final state $x$ of the local search
trajectory,
which is hard to directly sample from.
The common strategy of optimizing the variational lower bound of likelihood
would require policy gradient~\citep{mnih2014neural}
and introduce extra samplers.
%which makes the overall effort of optimizing variational form of power
%iteration in vain.
Instead,
inspired by~\citet{steinhardt2015learning},
we use importance sampling with self-normalization to estimate the gradient
in \eqref{eq:grad_q}.
Specifically, let $s_x(\xb_{0:t-1})$ be the proposal distribution of the
local search trajectory. 
We then have
\begin{equation}
	\nabla_{\phi} \log q(x;\phi) = \EE_{s_x(\xb_{0:t-1})} \sbr{ \frac{q(\xb_{0:t} | x_t = x; \phi)}{s_x(\xb_{0:t-1})} \nabla_{\phi} \log q([\xb_{0:t-1}, x]; \phi) }
\end{equation}
In practice, we draw $N$ trajectories from the proposal distribution,
and approximate the normalization constant in $q(\xb_{0:t} | x_t = x; \phi)$
via self-normalization. The Monte Carlo gradient estimator is:
\begin{eqnarray}
    \nabla_{\phi} \log q(x; \phi) & \simeq & \frac{1}{N} \sum_{j=1}^N \frac{q(\xb^j_{0:t^j}|x_{t^j}=x; \phi)}{s_x(\xb^j_{0:t^j-1})} \nabla_{\phi} \log q([\xb_{0:t^j-1}^j, x]; \phi) \nonumber \\
    & \simeq & \frac{1}{N} \sum_{j=1}^N \frac{q(\xb^j_{0:t^j}; \phi)}{s_x(\xb^j_{0:t^j-1}) \sum_{k=1}^N q(\xb^k_{0:t^k}; \phi) } \nabla_{\phi} \log q_{\phi}([\xb_{0:t^j-1}^j, x]; \phi)
    \label{eq:monte_carlo_grad}
\end{eqnarray}
The self-normalization trick above is also equivalent to re-using the same
proposal samples from $s_x(\xb_{0:t-1})$ to estimate the normalization term
in the posterior $q(\xb_{0:t} | x_t = x; \phi)$.
Then,
given a sample $x$, a good proposal distribution for trajectories
needs to guarantee that every proposal trajectory ends exactly at $x$.
Below we propose two designs for such proposal.   

\noindent\textbf{Inverse proposal:}
Instead of randomly sampling a trajectory and hoping it arrives exactly at $x$,
we can walk backwards from $x$, sampling $x_{t-1}, x_{t-2}, \ldots, x_0$.
We call this an inverse proposal.
In this case, we first sample a trajectory length $t$.
Then for each backward step, we sample $x_k \sim A'(x_k|x_{k+1})$.
For simplicity, we sample $t$ from a truncated geometric distribution,
and choose $A'(\cdot|\cdot)$ from the same distribution family as the forward
editor $q_A(\cdot|\cdot)$, except that $A'$ is not trained.
In this case we have 
\begin{equation}
	s_x(\xb_{0:t-1}) = \text{Geo}(t) \prod_{i=0}^{t-1} A'(x_i|x_{i+1}) 
\label{eq:inv_proposal}
\end{equation}
Empirically we have found that the learned local search sampler
will adapt to the energy model with a different expected number of edits,
even though the proposal is not learned.

\noindent\textbf{Edit distance proposal:}
In cases when we have a good $q_0$,
we design the proposal distribution based on shortest edit distance.
Specifically, we first sample $x_0 \sim q_0$.
Then, given $x_0$ and the target $x$, we sample the trajectory
$\xb_{1:t-1}$ that would transform $x_0$ to $x$
with the minimum number of edits.
For the space of discrete data $\Scal = \cbr{0, 1}^d$, the number of edits
equal the hamming distance between $x_0$ and $x$;
if $S$ corresponds to programs,
then this corresponds to the shortest edit distance.
Thus
\begin{equation}
	s_x(\xb_{0:t-1}) \propto q_0(x_0) \II\sbr{t = \text{ShortestEditDistance}(x_0, x)}
\label{eq:short_proposal}
\end{equation}
Note that such proposal only has support on shortest paths, which would give a biased gradient in learning the local search sampler. In practice, we found such proposal works well. If necessary, this bias can be removed: For learning the EBM, we care only about the distribution over end states, and we have the freedom to design the local search editor, so we could limit it to generate only shortest paths, and the edit distance proposal would give unbiased gradient estimator.

\noindent\textbf{Parameterization of $q_A$:}
We restrict the editor $q_A(\cdot|x_{i-1})$ to make local modifications,
since local search has empirically strong performance~\citep{chen2019learning}.
Also such transitions resemble Gibbs sampling, which introduces a good
inductive bias for optimizing the variational form of power iteration.
Two example parameterizations for the local editor are:
\begin{itemize}[leftmargin=*,nolistsep,nosep]
	\item If $x \in \cbr{0, 1, \ldots, K}^d$,
then $q_A(\cdot|x_{i-1}) = \text{Multi}(d)\times\text{Multi}(K)$,
where the first multinomial distribution decides a position to change,
and the second one chooses a value for that position. 
	\item If $x$ is a program, the editor chooses a statement in
 the program and replaces with a generated statement.
The statement selector follows a multinomial with arbitrary dimensionality
using the pointer mechanism~\citep{vinyals2015pointer},
while the statement generater can be an autoregressive tree generator. 
\end{itemize}

The learning algorithm for the sampler and the overall learning
framework is summarized in \figref{fig:algo}. Please also refer to our open sourced implementation for more details~\footnote{\url{\url{https://github.com/google-research/google-research/tree/master/aloe}}}.

% !TEX root = main.tex
\section{Related work}
\label{sec:related}

\noindent\textbf{Learning EBMs:} 
Significant progress has recently been made in 
learning \textit{continuous} EBMs~\citep{du2019implicit, yu2020training},
thanks to efficient MCMC algorithms with gradient
guidance~\citep{neal2011mcmc, welling2011bayesian}.
Interestingly,
by reformulating contrastive learning as a minimax
problem~\citep{kim2016deep, dai2019exponential, arbel2020kale},
in addition to the model~\citep{dai2018kernel} 
these methods also learn a sampler that can generate realistic 
data~\citep{nijkamp2019learning}.
However learning the sampler and %HMC/SGLD
gradient based MCMC
require the existence of the gradient with respect to data points,
which is unavailable for discrete data. 
These methods
can also be adapted to discrete data using policy gradient,
but might be unstable during optimization. % as investigated in next section. 
Also for continuous data,
~\citet{xie2018cooperative} proposed an MCMC teaching framework that shares
a similar principle to our variational power method, when the number of power 
iterations is limited to 1 (\algref{alg:learn_q}).
Our work is different in that we propose a local search sampler and novel importance proposal 
that is more suitable for discrete spaces of structures.

For \textit{discrete} EBMs, classical methods like CD, PCD or wake-sleep
are applicable,
but with drawbacks (see~\secref{sec:variational_pi}).
Other recent work with discrete data includes
learning MRFs with a variational upper
bound~\citep{kuleshov2017neural},
using the Gumbel-Softmax trick~\citep{Li2020To},
or using a residual-energy 
model~\citep{bakhtin2020energybased, deng2020residual}
with a pretrained proposal for noise contrastive
estimation~\citep{gutmann2010noise},
but these are not necessarily suitable for general EBMs.
SPEN~\citep{belanger2016structured, belanger2017end}
proposes a continuous relaxation %of discrete data 
combined with a max-margin 
principle~\citep{taskar2004max},
which works well for structured prediction,
%but might not capture the distribution due to a 
but could suffer from mode collapse.
%issues similar to GANs~\citep{goodfellow2014generative}.

\noindent\textbf{Learning to search:}
Our parameterization of the negative sampler with auxiliary-variable local
search is also related to work on learning to search.
Most work in that literature considers learning the search strategy 
given demonstrations~\citep{he2014learning,chang2015learning,song2018learning,guez2018learning}.
When no supervised trajectories are available, 
%for the search trajectory, 
policy gradient with variance reduction is typically used to improve the
search policy,
 in domains like machine translation~\citep{xia2017deliberation}
and combinatorial optimization~\citep{chen2019learning}.
Our variational form for power iteration circumvents the need for REINFORCE,
and thereby gains significant stability in practice.
%which is more stable in practice. 
% The technique of importance reweighted gradient estimator is inspired by~\citet{steinhardt2015learning}. We extend it with two novel proposals that are more suitable for modifying subparts of structured data in local search.
%\bodai{not necessary, we have cite them in previous section.}

\noindent\textbf{Other discrete models:}
%Besides EBMs, 
There are many other models for discrete data, like invertible flows for sequences~\citep{tran2019discrete, hoogeboom2019integer} or graphs~\citep{shi2020graphaf, madhawa2019graphnvp}. 
Recently there is also interest in learning non-autoregressive models for
NLP~\citep{stern2019insertion, gu2019levenshtein, ghazvininejad2019constant}.
%which shares some of the same evidence as our proposed local search sampler.
The main focus of ALOE is to provide a new learning algorithm for EBMs.
%A thorough comparison between 
Comparing
EBMs and other discrete models will be 
interesting for future investigation.

\section{Experiments}

\label{sec:experiment}

%We experimentally evaluate ALOE on three tasks:
%synthetic problems in~\secref{sec:exp_synthetic},
%generative fuzzing for software testing in~\secref{sec:exp_fuzzing},
%and program synthesis in~\secref{sec:exp_robustfill}.
% These experiments evaluate both the learned energy model and the sampler
% quantitatively and qualitatively.  

\subsection{Synthetic problems}
\label{sec:exp_synthetic}

%The synthetic experiment focuses 
We first focus
on learning unconditional discrete EBMs $p(x) \propto \exp f(x)$ 
from data with an unknown distribution, where 
the data consists of bit vectors
$x \in \cbr{0, 1}^{32}$.

\noindent\textbf{Baselines:}
We compare against a hand designed sampler and a learned sampler 
%for discrete EBM learning 
from the recent literature. 
The hand designed sampler baseline is PCD~\citep{tieleman2008training} 
using a replay buffer and random restart tricks~\citep{du2019implicit},
which has shown superior results in image generation.
The learned sampler baseline is the discrete version of 
ADE~\citep{dai2019exponential}. 
Please refer to \appref{app:exp_synthetic} for more details about the
baseline setup. 

\noindent\textbf{Experiment setup:}
% We design the experiment in a way that can be evaluated 
This experiment is designed to allow
both a quantitative and 2D visual evaluation.
We first collect synthetic 2D data in a continuous space 
\citep{grathwohl2018ffjord},
where the 2D data $\hat{x} \in \RR^2$ is sampled from some unknown
distribution $\hat{p}$.
For a given $\hat{x}$, we convert the floating-point number representation
(with precision up to $1e^{-4}$) of each dimension into a 16-bit Gray 
code.\footnote{\url{https://en.wikipedia.org/wiki/Gray\_code}}
This means the unknown true distribution in discrete space is
$p(x) = \hat{p}( [\texttt{GrayToFloat}(x_{0:15})/1e^4, \texttt{GrayToFloat}(x_{16:31})/1e^4])$.
This task is challenging even in the original 2D space,
compounded by the nonlinear Gray code.
%Given that the Gray code is nonlinear compared to the natural binary codes,
%this task is highly nontrivial.
%
All the methods learn the same score function $f$,
which is parameterized by a 4-layer MLP with ELU~\citep{clevert2015fast}
activations and hidden layer size$=256$.
ADE and ALOE learns the same form of $q_0$. Since the dimension is fixed to 32, $q_0$ is an autoregressive model with no parameter sharing across 32 steps. For ALOE we also use Gibbs sampling as the base MCMC sampler, but we only perform one pass over 32 dimensions, which is only $1/10$ of what PCD used.

\begin{figure*}
\centering
\setlength{\tabcolsep}{0.1em}
\begin{tabular}{ccccccc}
	\includegraphics[width=0.14\textwidth]{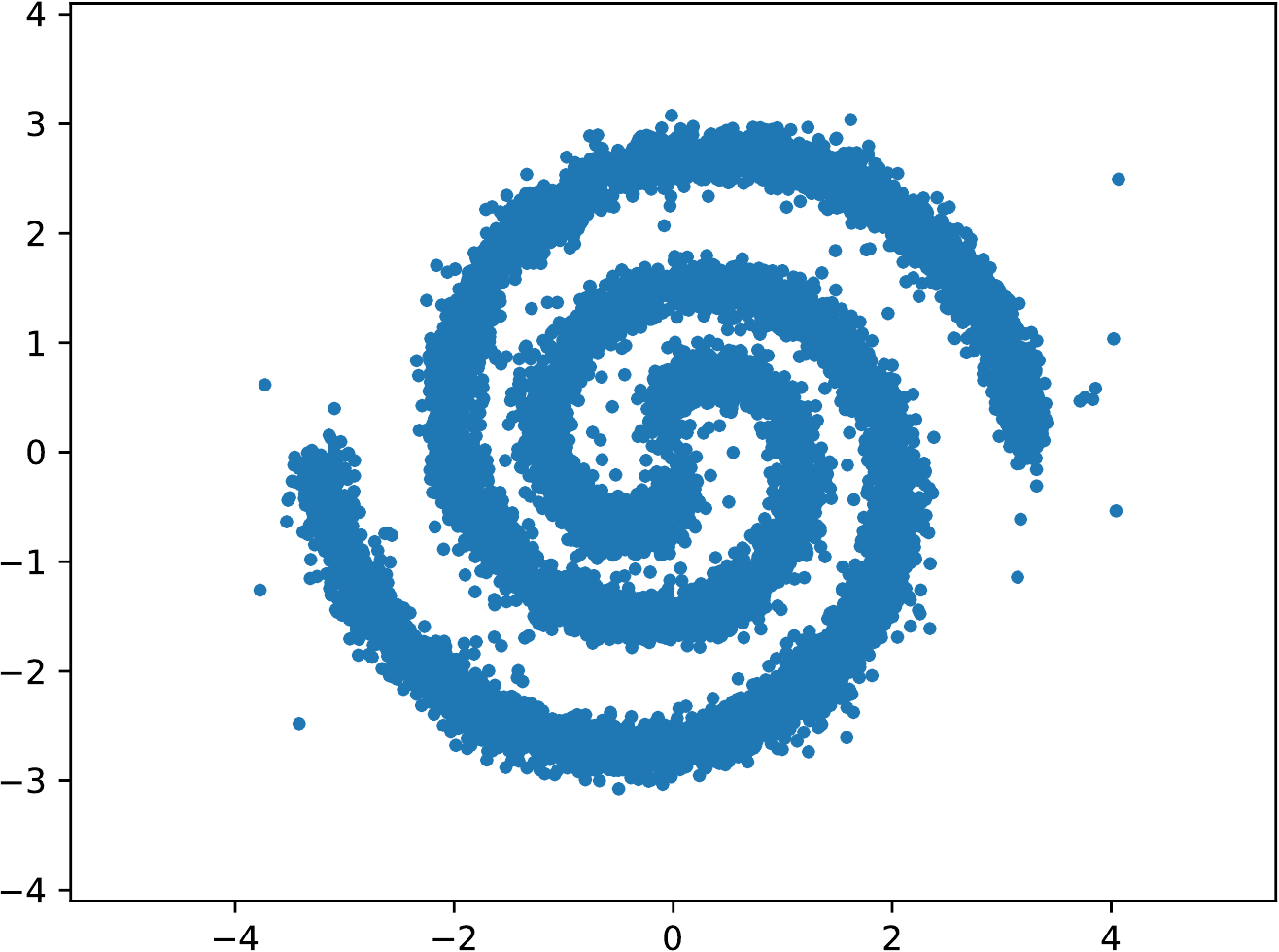} &
	\includegraphics[width=0.14\textwidth]{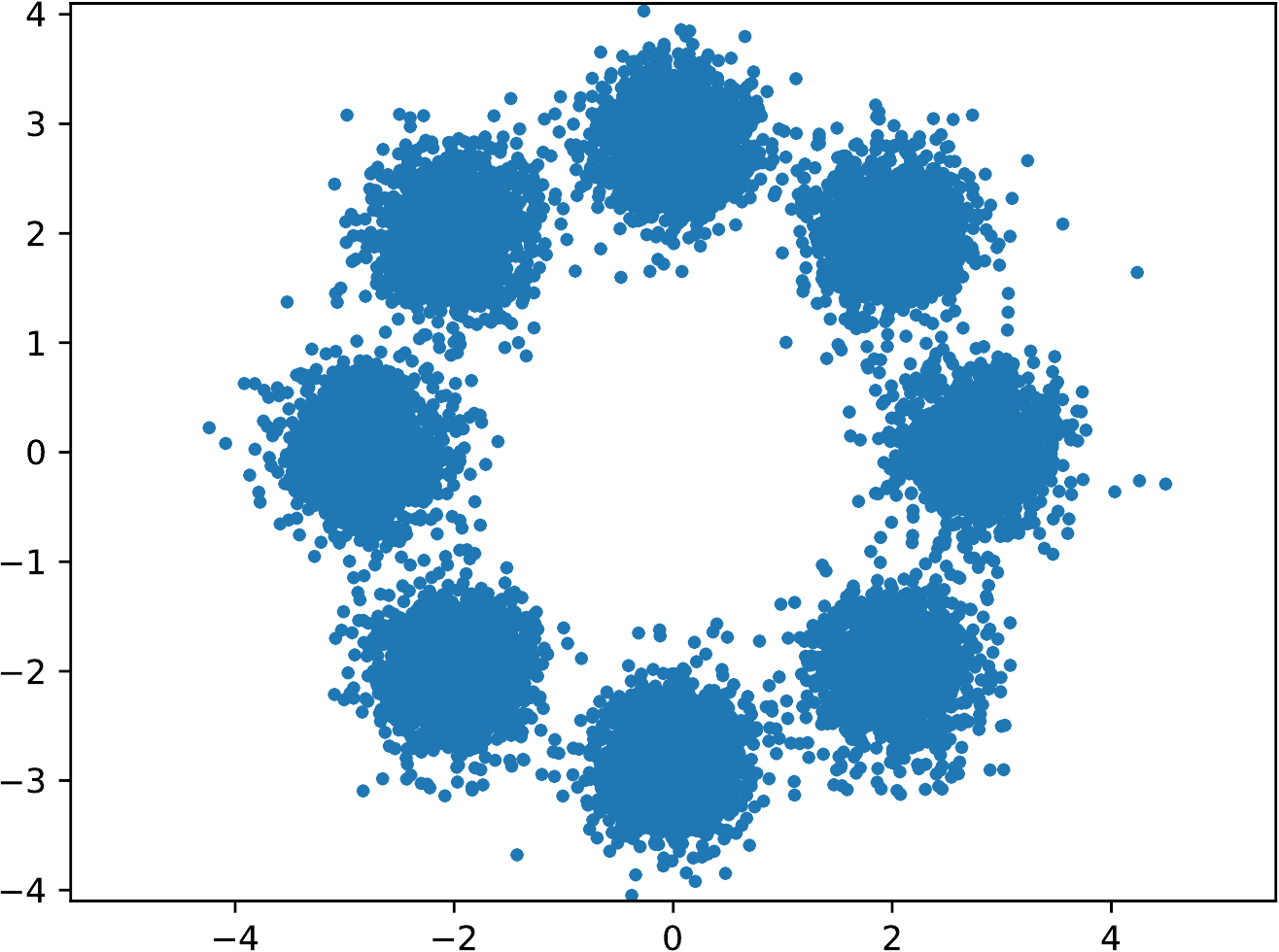} &
	\includegraphics[width=0.14\textwidth]{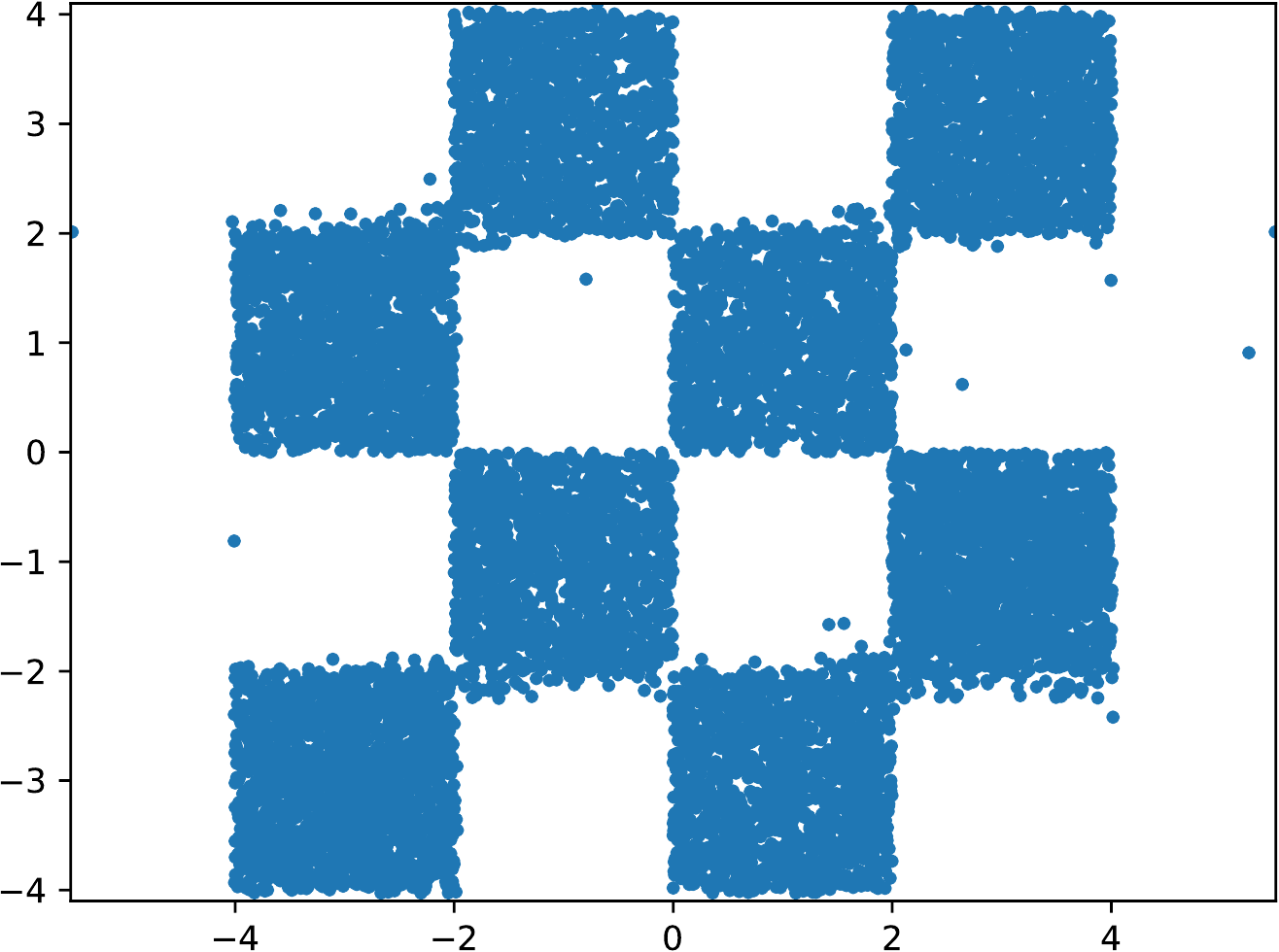} &
	\includegraphics[width=0.14\textwidth]{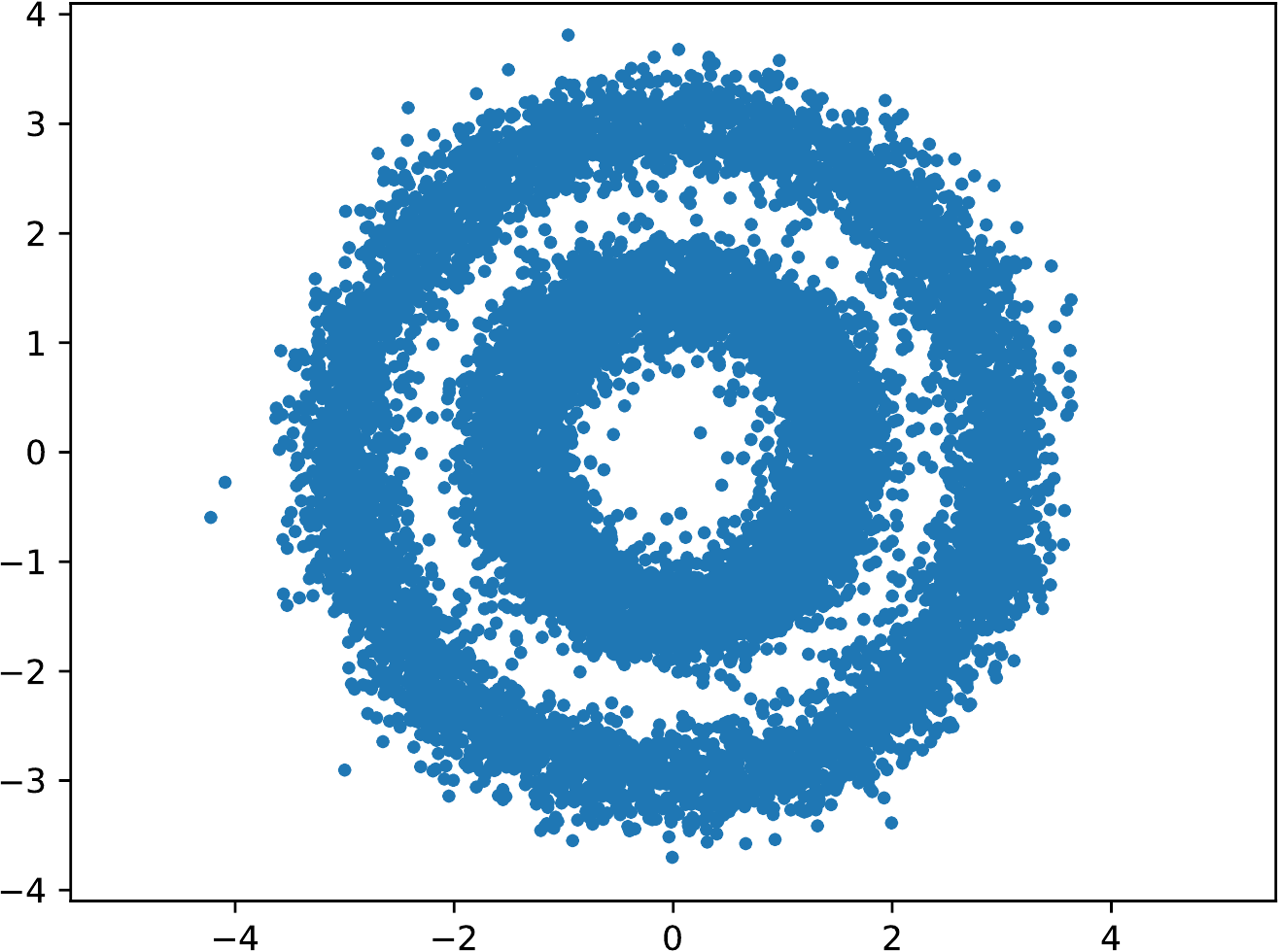} &
	\includegraphics[width=0.14\textwidth]{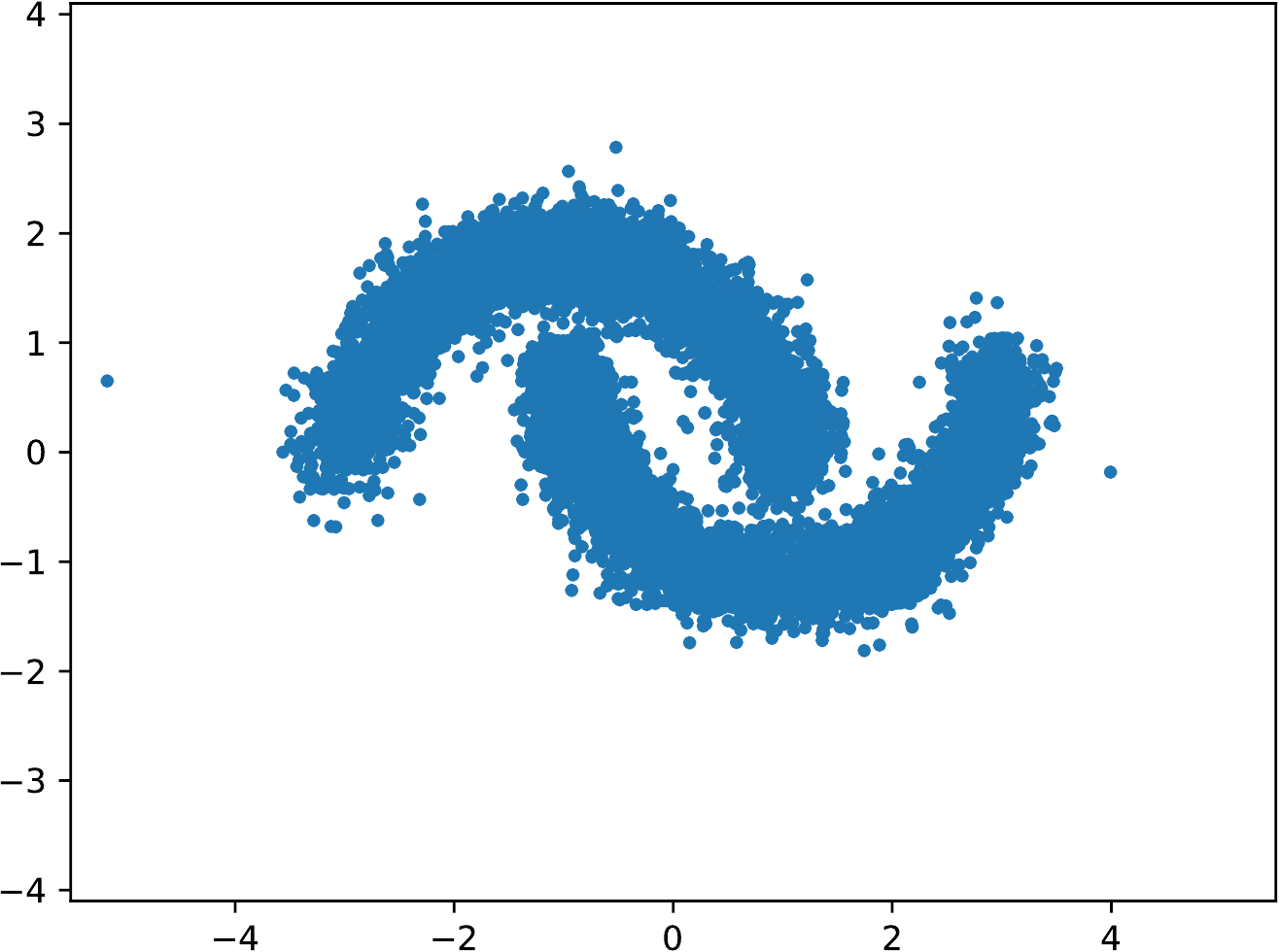} &
	\includegraphics[width=0.14\textwidth]{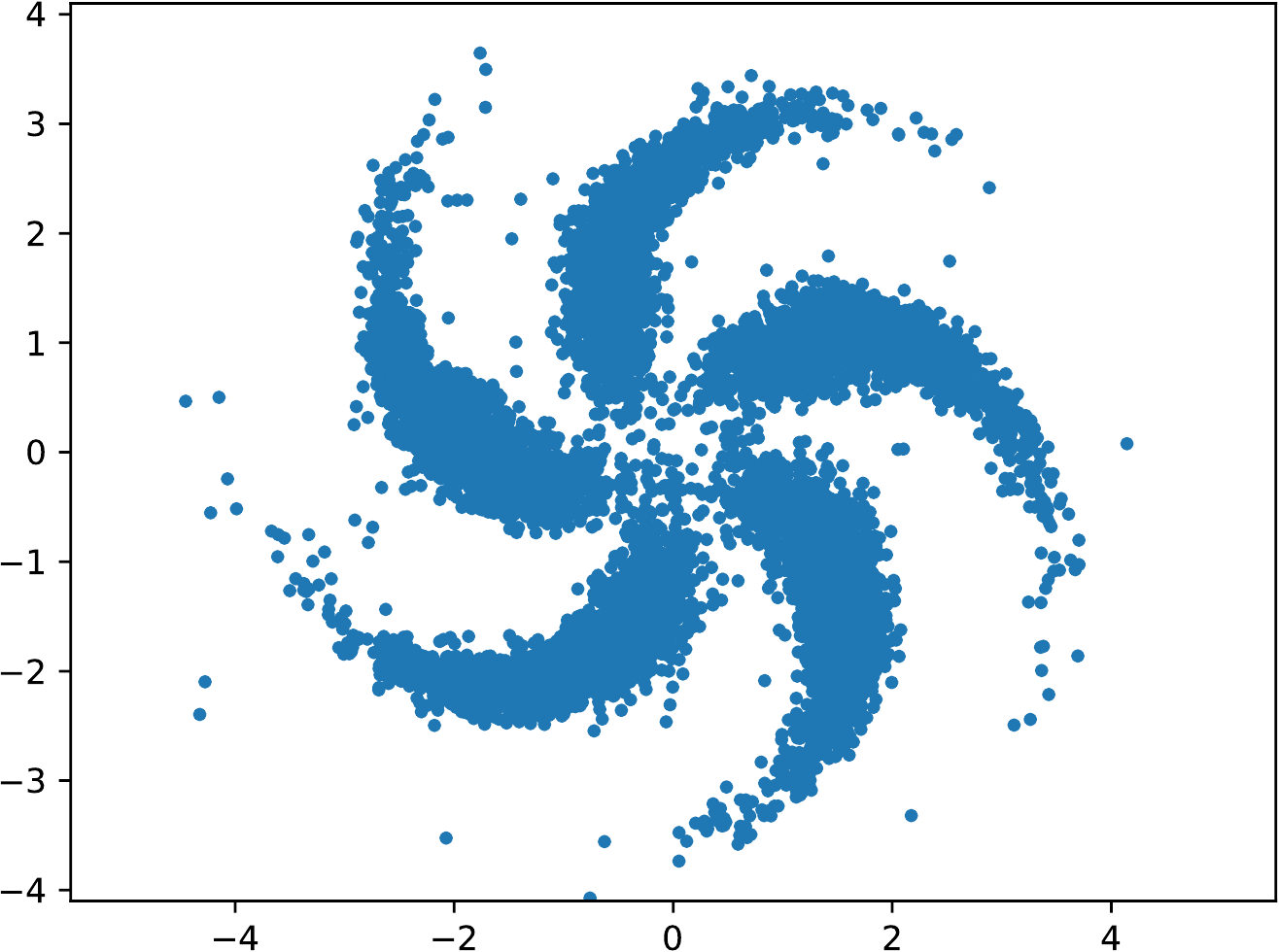} &
	\includegraphics[width=0.14\textwidth]{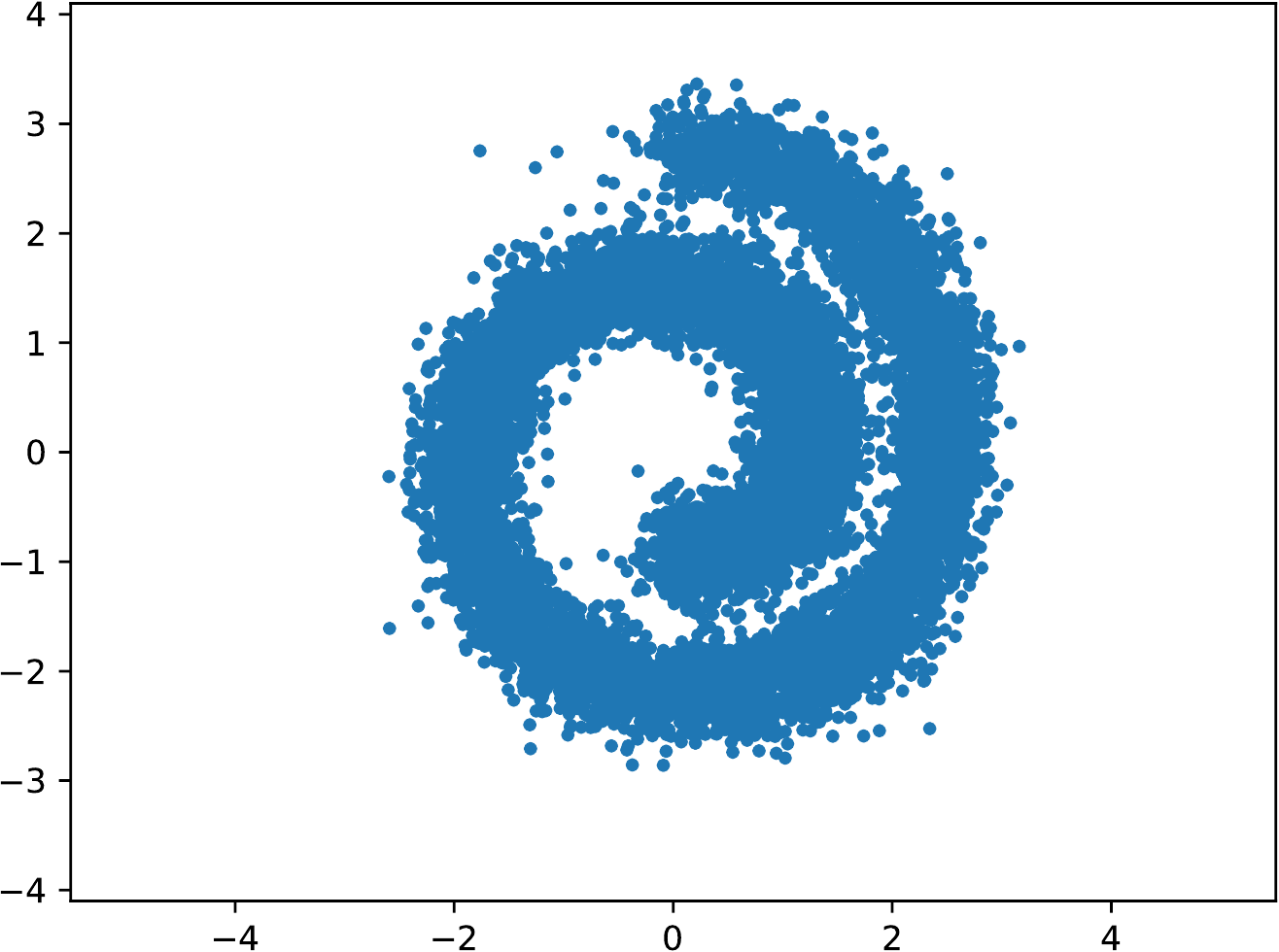}  
	\\
	\includegraphics[width=0.14\textwidth]{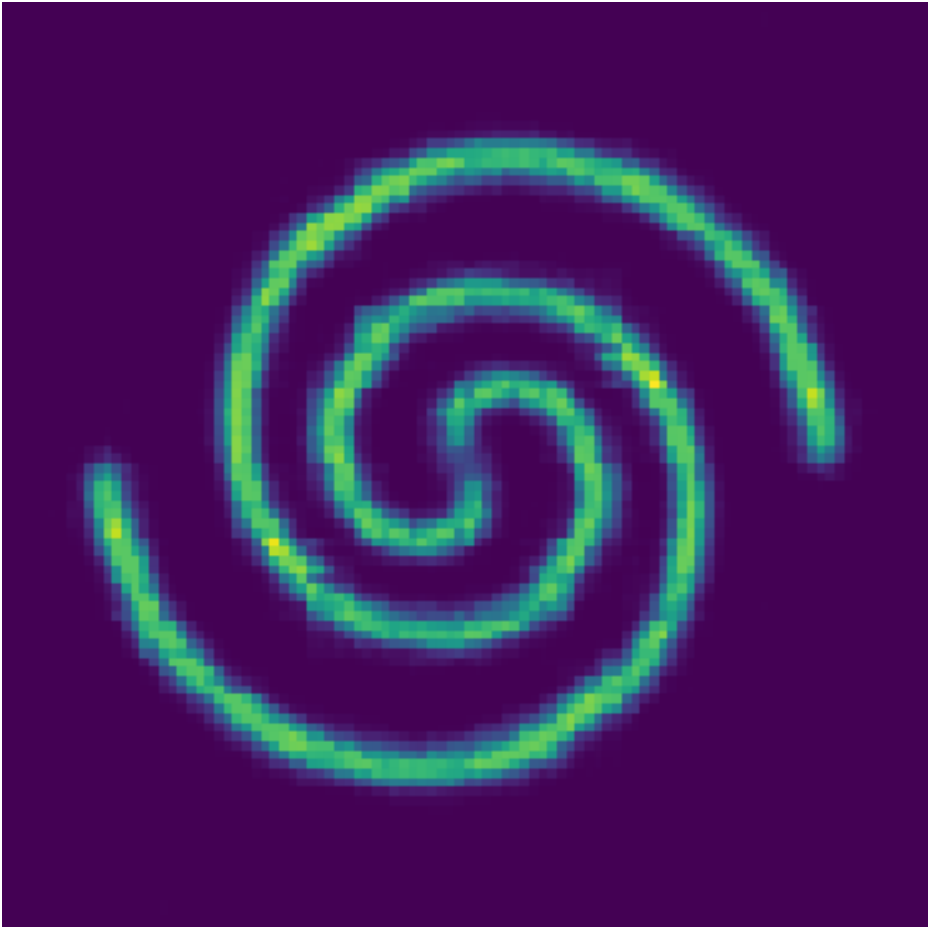} & 
	\includegraphics[width=0.14\textwidth]{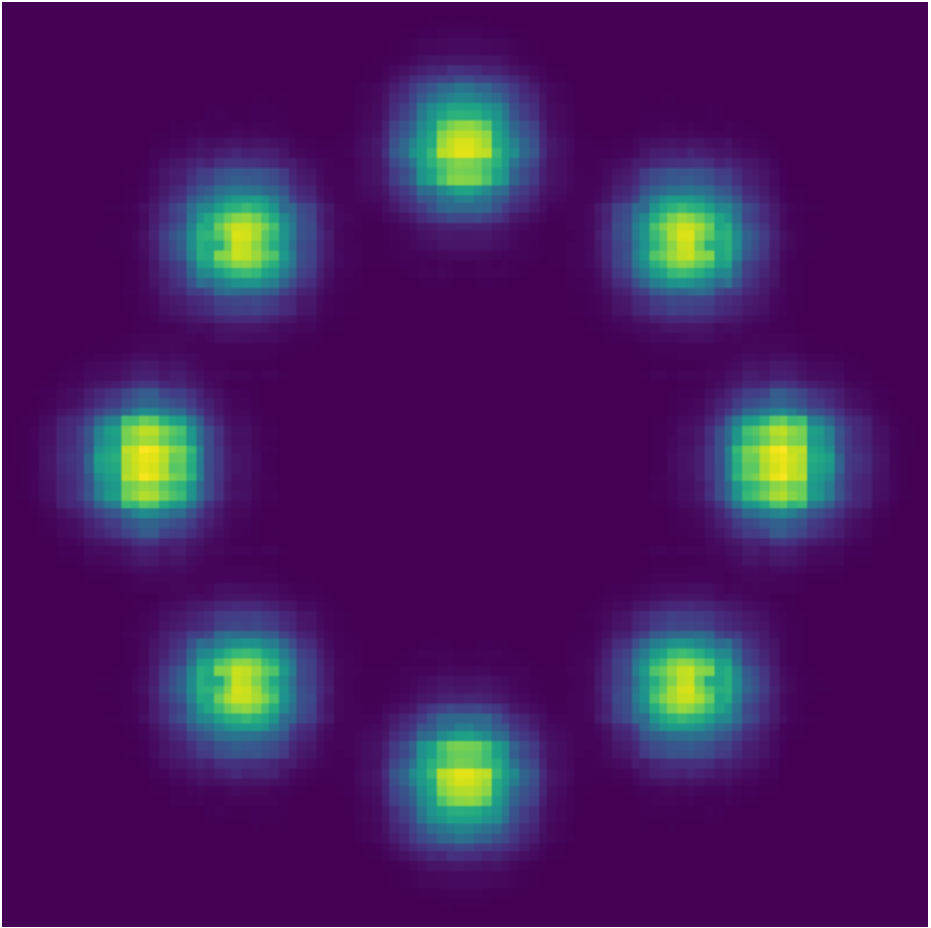} & 
	\includegraphics[width=0.14\textwidth]{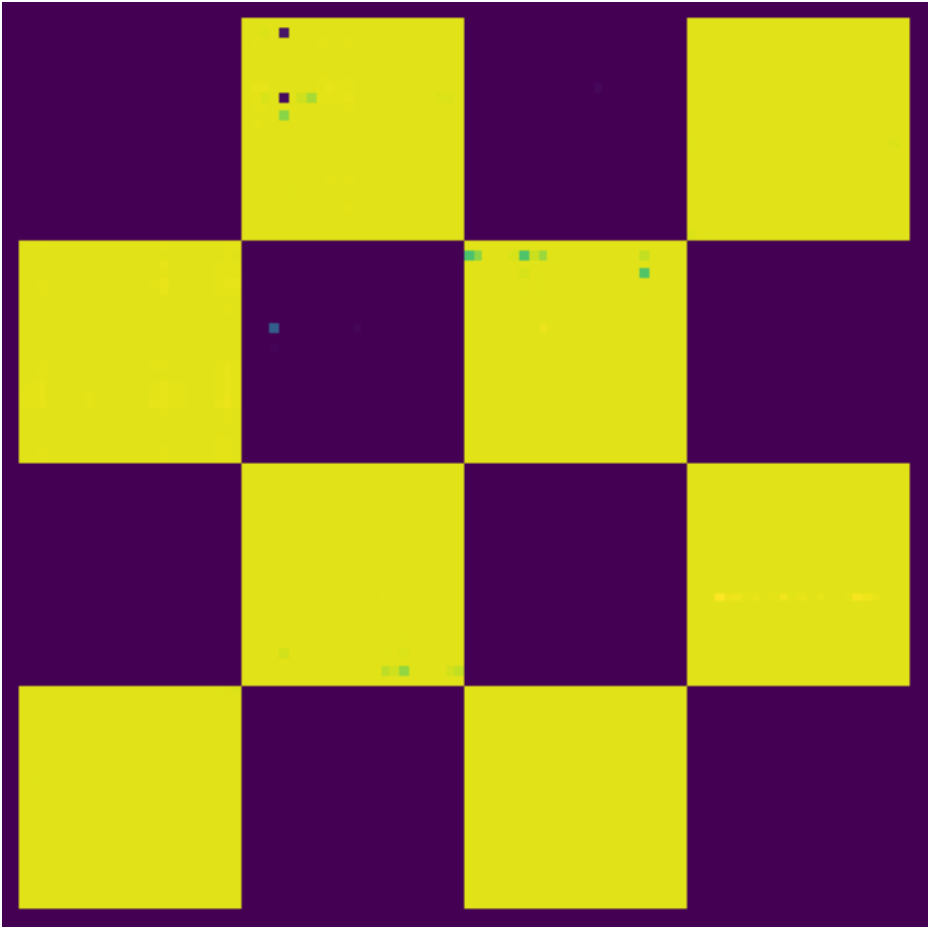} & 
	\includegraphics[width=0.14\textwidth]{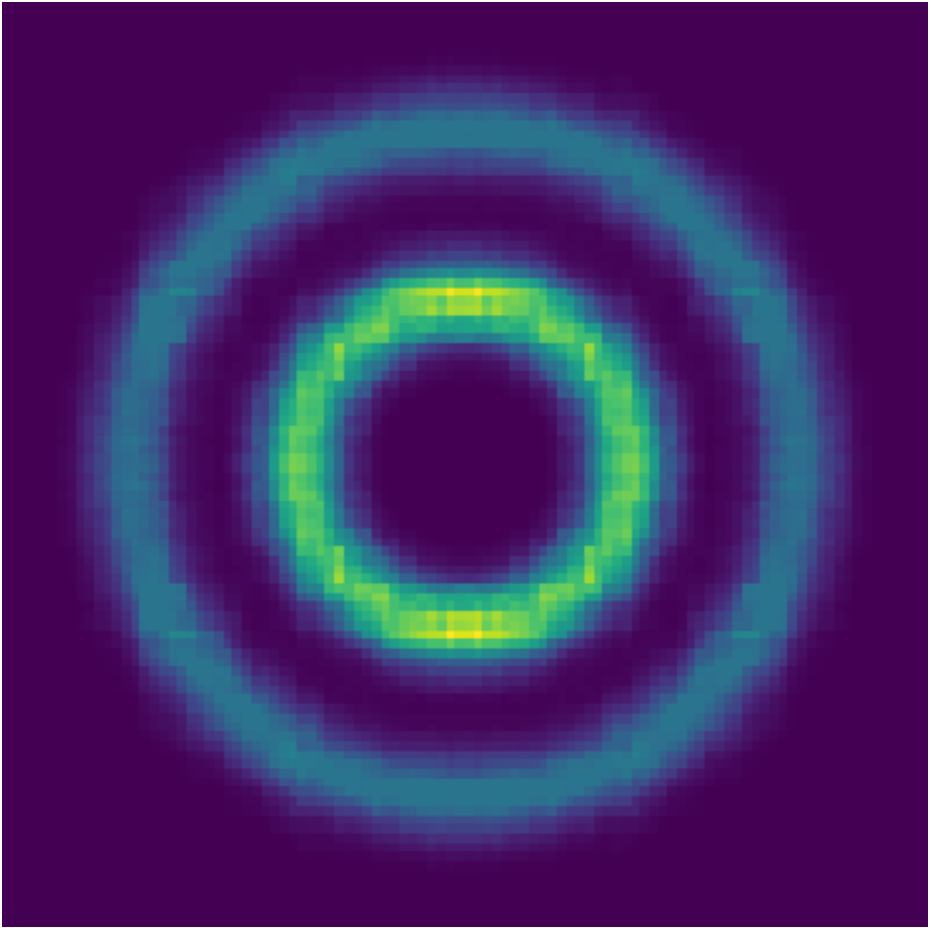} & 
	\includegraphics[width=0.14\textwidth]{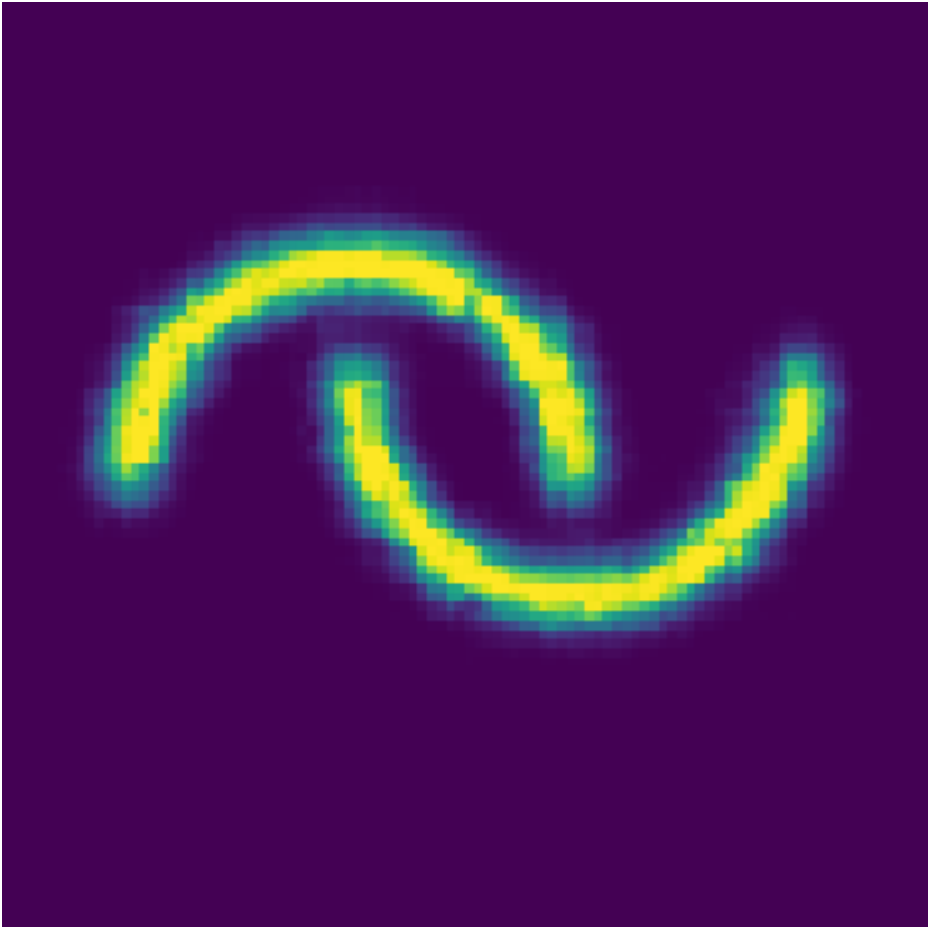} & 
	\includegraphics[width=0.14\textwidth]{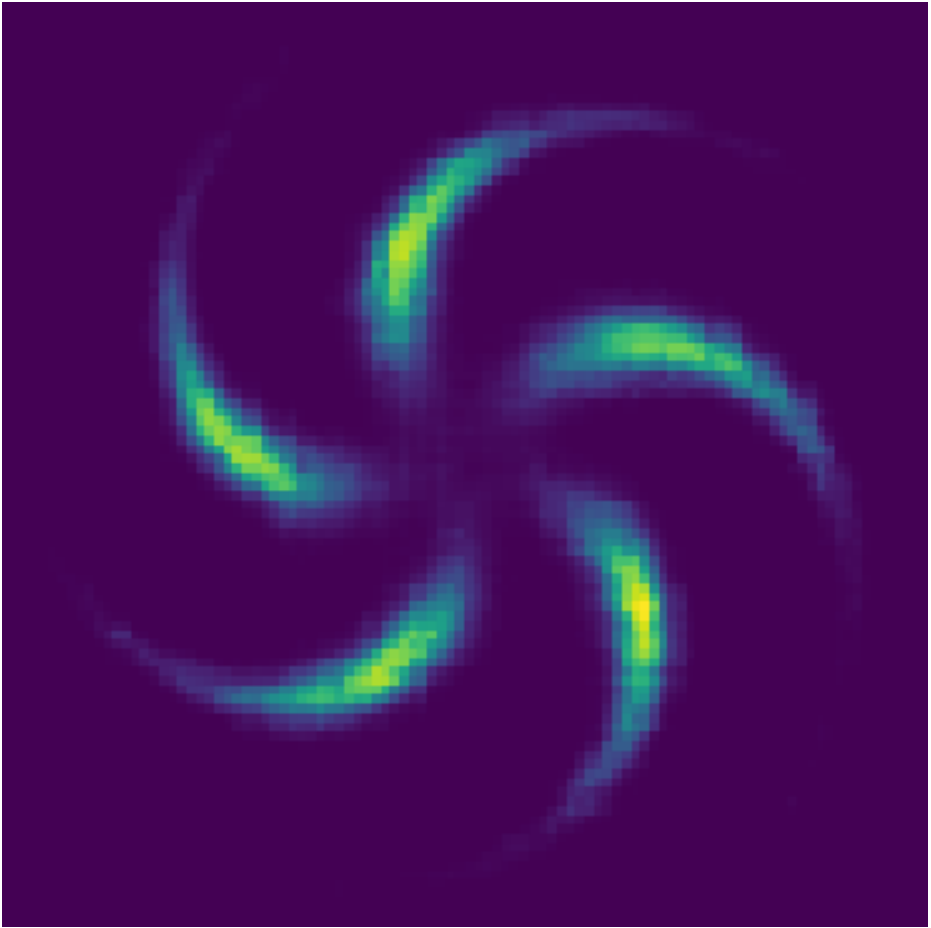} &
	\includegraphics[width=0.14\textwidth]{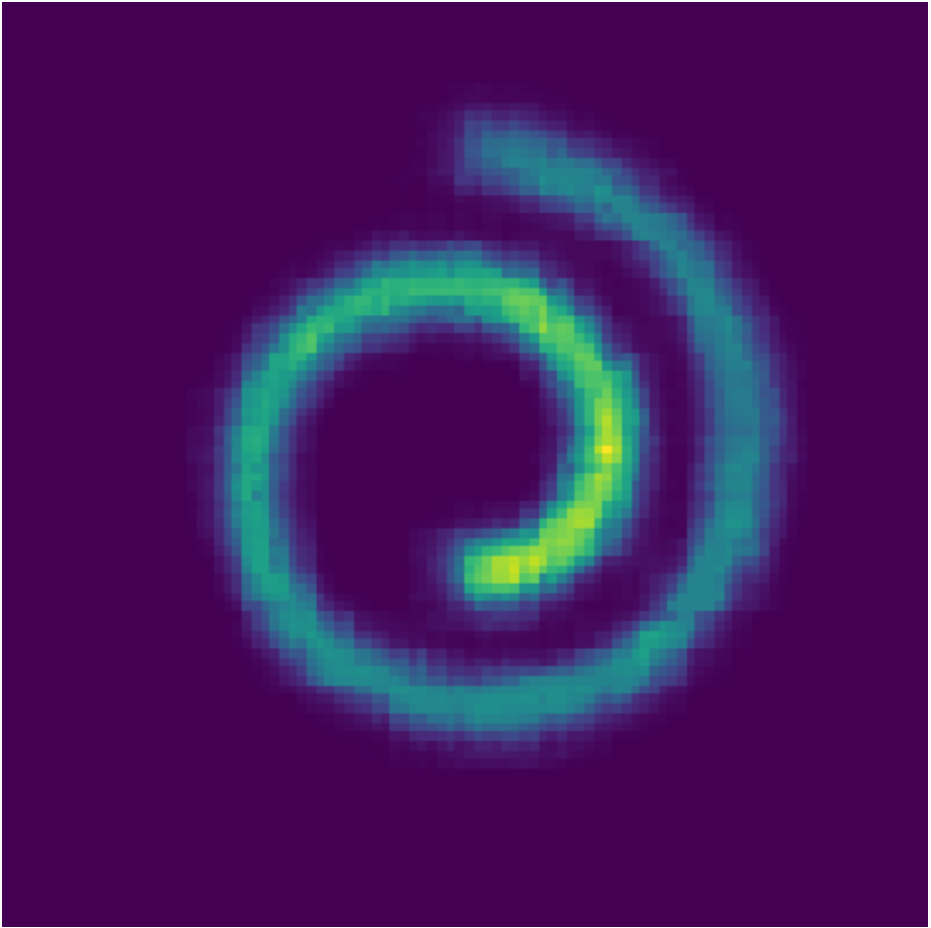}
\end{tabular}
\caption{Visualization of learned energy model and sampler. From left to right: 2spirals, 8gaussians, checkerboard, circles, moons, pinwheel, swissroll. Due to the limited space, please refer to \figref{fig:sample_gt} in appendix for the visualization of training samples. \label{fig:synthetic_vis}}
\end{figure*}

\begin{figure*}
\centering
\setlength{\tabcolsep}{0.5em}
\begin{tabular}{@{}ccc}
\includegraphics[width=0.33\textwidth]{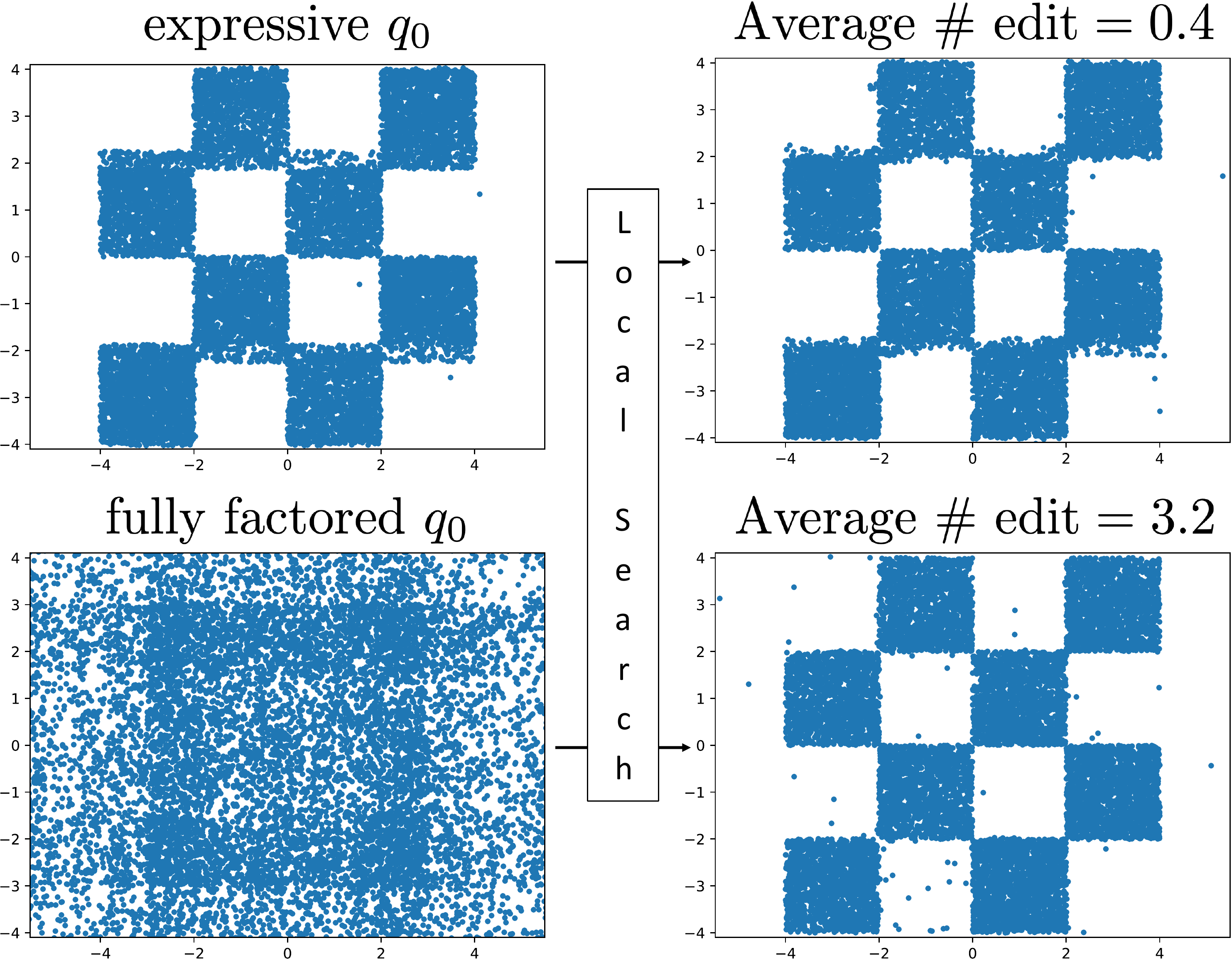} & 
\includegraphics[width=0.29\textwidth]{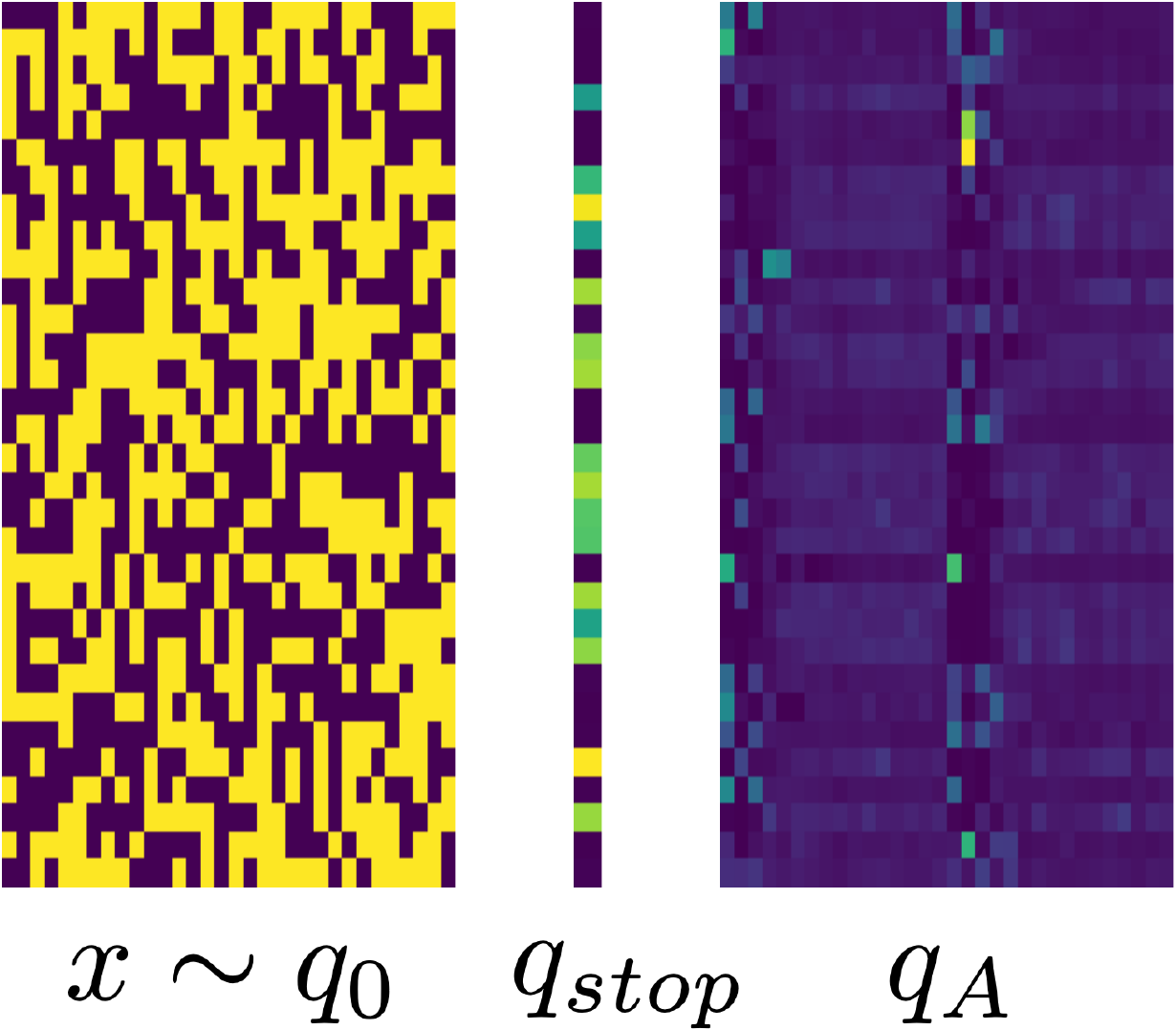} & 
\includegraphics[width=0.29\textwidth]{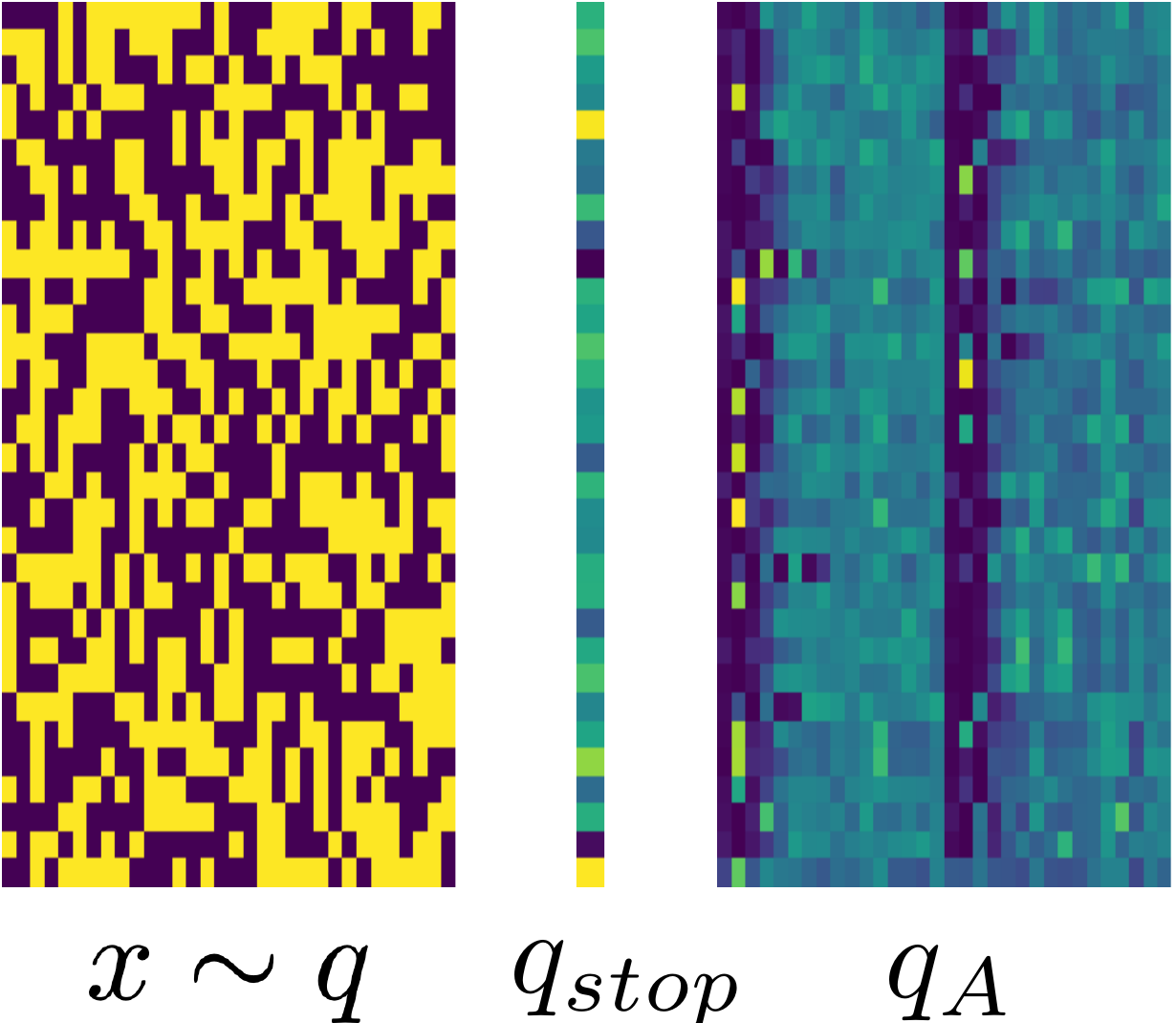} \\
Sampler visualization in 2D space & Run edit on initial $x_0$ & Run edit on $x$ from sampler 
\end{tabular}
\caption{Visualization of learned local search sampler in 2D (left) and original discrete Gray code (mid + right) space. See~\secref{sec:exp_synthetic} for more information. }
\label{fig:vis_q}
\end{figure*}

\begin{table}[t]
\centering
\caption{Synthetic results with MMD-hamming {($\mathbf{\times 1e^{-3}}$)} as evaluation metric, and the lower the better. * denote the discrete adaptation of its original method for continuous domain. \label{tab:synthetic}}
\renewcommand{\arraystretch}{1.1}
\resizebox{1.0\textwidth}{!}{%
\begin{tabular}{ccccccccc}
	\toprule
	& & 2spirals & 8gaussians & circles & moons & pinwheel & swissroll & checkerboard \\
	\hline
	& PCD-10*~\citep{tieleman2008training,du2019implicit} & 34.73 & 0.3 & -0.3 & 0.48 & -0.42 & -0.49 & -1.04 \\
	& ADE*~\citep{dai2019exponential} & 33.4 & -0.28 & 2.01 & 2.16 & 7.64 & 6.12 & -0.69 \\
	& \modelshort{} & {\bf 30.37} & {\bf -0.97} & {\bf -0.83} & {\bf -0.64} & {\bf -0.64} & {\bf -0.58} & {\bf -1.7} \\
	\hline
	\hline
\parbox[t]{2mm}{\multirow{3}{*}{\rotatebox[origin=c]{90}{Ablation}}} & ADE-fac & 236.6 & 65.7 & 261.7 & 248.6 & 187.2 & 95.3 & 78.2 \\
& ALOE-fac-noEdit & 51.24 & 91.2 & 5.97 & 76.8 & 59.7 & 15 & 2.98 \\
& ALOE-fac-edit & 32.6 & 3 & {\bf -1.5} & 1.27 & 5.02 & 0.44 & {\bf -2.03} \\
\hline
\hline
\parbox[t]{2mm}{\multirow{2}{*}{\rotatebox[origin=c]{90}{Other}}} &  AutoRegressive & 32.7 & -0.3 & -0.8 & -0.45 & {\bf -1.27} & 0.31 & -0.2 \\
& VAE & 35.2 & 2.09 & 0.16 & 1.1 & 0.85 & 2.05 & -0.77 \\
	\bottomrule
\end{tabular}
}
\end{table}

\noindent\textbf{Main results:}
To quantitatively evaluate different methods, we use MMD~\citep{gretton2012kernel} with linear kernel (which corresponds to $32 - \text{HammingDistance}$) to evaluate the empirical distribution between true samples and samples from the learned energy function. To obtain samples from $f$, we run $20\times32$ steps of Gibbs sampling and collect 4000 samples. We can see from~\tabref{tab:synthetic} that ALOE consistently outperforms alternatives across all datasets. ADE variant is worse than PCD on some datasets, as REINFORCE based approaches typically requires careful treatment of the gradient variance. 

We also use VAE~\citep{kingma2013auto} or autoregressive model to learn the discrete distribution, where the results are shown in the ``Other'' section of~\tabref{tab:synthetic}. Note that these models are different, so the numerical comparison is mainly for the sake of completeness. \modelshort{} mainly focuses on learning a given energy based model, rather than proposing a new probabilistic model.

\noindent\textbf{Visualization:}
We first visualize the learned score function and sampler in~\figref{fig:synthetic_vis}. To plot the heat map, we first uniformly sample 10k 2D points with each dimension in $[-4, 4]$. Then we convert the floating-point numbers to Gray code to evaluate and visualize the score under learned $f$. Please refer to \appref{app:exp_synthetic} for more visualizations about baselines. 
In \figref{fig:vis_q}, we visualize the learned local search sampler in both discrete space and 2D space by decoding the Gray code. We can see ALOE gets reasonable quality even with a weak  $q_0(x) = \prod_{i=1}^{32} q_0(x[i])$, and it automatically adapts the refinement steps according to the quality of $q_0$. 

\begin{wrapfigure}{r}{0.58\textwidth}
\vspace{-5mm}
\begin{tabular}{@{}c@{}c@{}}
	\includegraphics[width=0.28\textwidth]{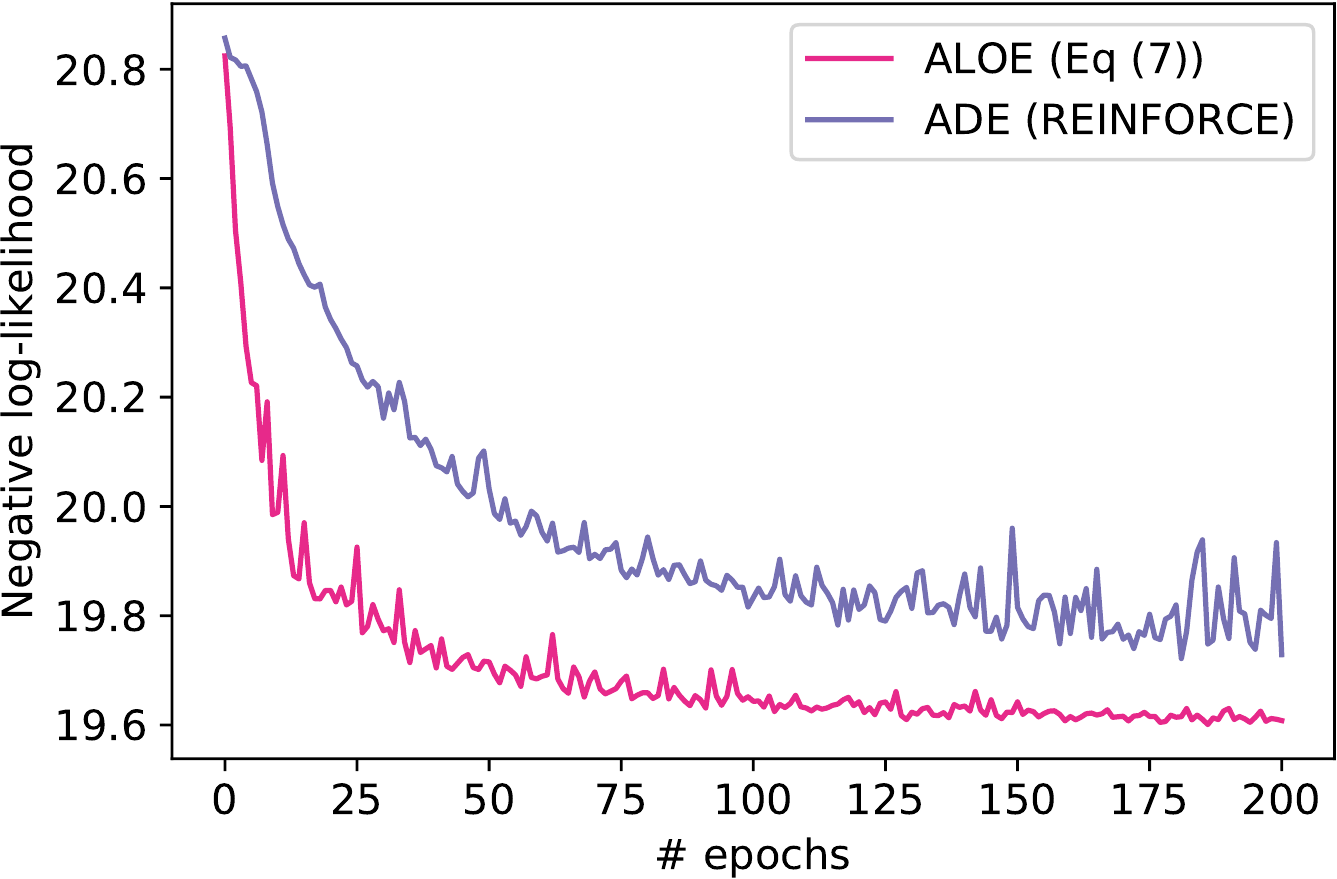} &
	\includegraphics[width=0.275\textwidth]{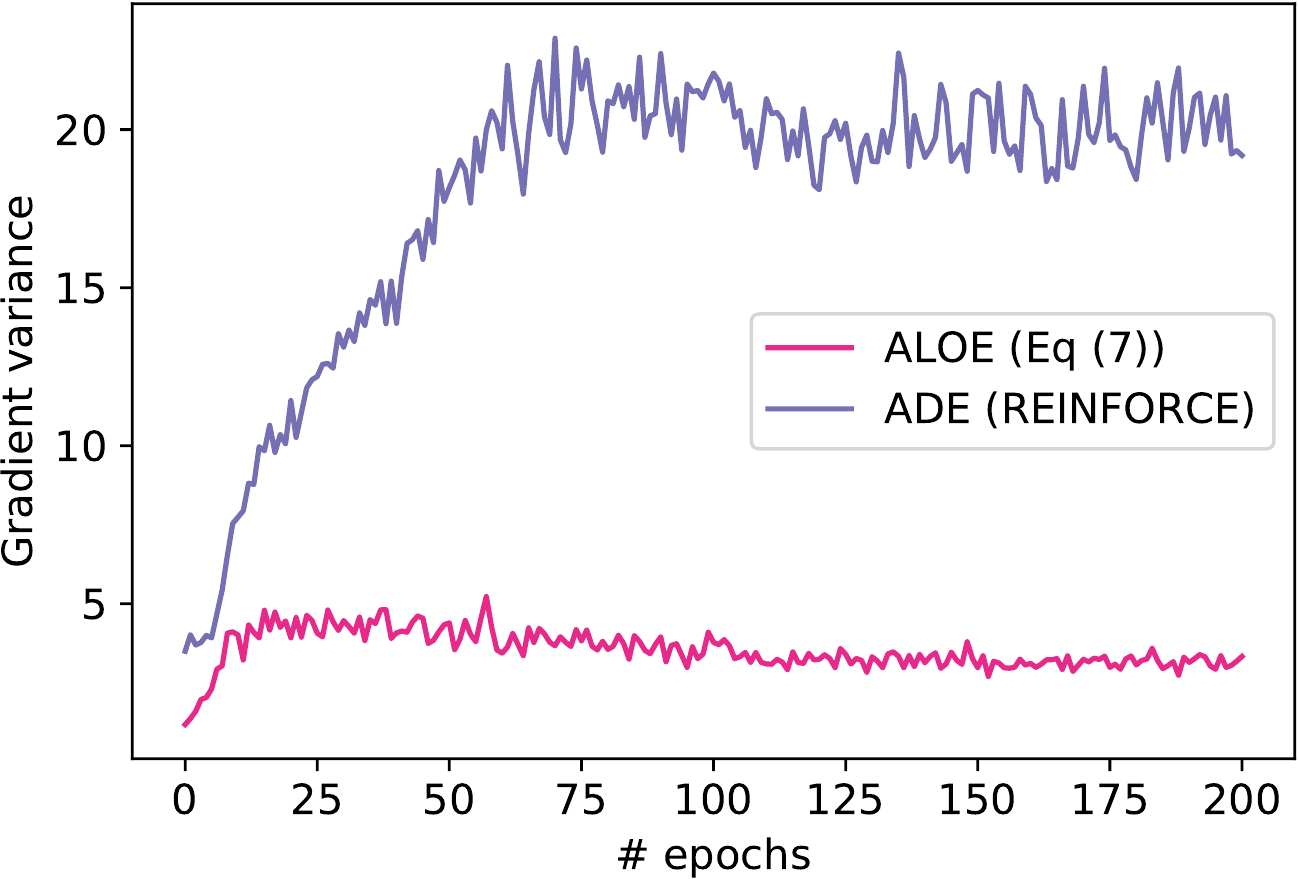}	 \\
	Negative Log-likelihood &
	Gradient Variance
\end{tabular}
\vspace{-1mm}
\caption{Training objective and gradient variance. \label{fig:grad_var} }
\vspace{-4mm}
\end{wrapfigure}

\noindent\textbf{Gradient variance:} Here we empirically justify the necessity of the variational power iteration objective design in \eqref{eq:variational_pi} against the REINFORCE objective. We train ADE and \modelshort{} (with only $q_0$ for comparison) on \texttt{pinwheel} data, and plot the negative log-likelihood of EBM (estimated via importance sampling) and the Monte Carlo estimation of gradient variance in \figref{fig:grad_var}. We can clearly see \modelshort{} enjoys lower variance and thus faster and better convergence than REINFORCE based methods for EBMs.

\noindent\textbf{Ablation:}
Here we try to justify the necessity of both local edits and the variational power iteration objective. \textbf{a)}To justify the local edits, we use a fully factorized initial $q_0$, and compare ALOE-fac-noEdit (no further edits) against ALOE-fac-edit (with $\leq$16 edits). 
ALOE-fac-edit performs much better than the noEdit version.  We use a weak $q_0$ here since we don't need many edits when $q_0$ is the powerful MLP with no parameter sharing (which is not feasible in realistic tasks). Nevertheless, ALOE automatically learns to adapt number of edits as studied in \figref{fig:vis_q} left. \textbf{b)}We also show the objective in \eqref{eq:variational_pi} achieves better results than the REINFORCE objective from ADE.

\vspace{-1mm}
\subsection{Generative fuzzing}
\label{sec:exp_fuzzing}
\vspace{-1mm}
\begin{figure*}[t]
\centering
%\setlength{\tabcolsep}{0.1em}
%\begin{tabular}{@{}ccc}
%	\includegraphics[width=0.33\textwidth]{figs/libpng-fuzz.pdf} &
%	\includegraphics[width=0.33\textwidth]{figs/openjpeg-fuzz.pdf} &
%	\includegraphics[width=0.33\textwidth]{figs/libmpeg2-fuzz.pdf} \\
%	\texttt{libpng} & 
%	\texttt{openjpeg} & 
%	\texttt{libmpeg2}
%\end{tabular}
%\vspace{-3mm}
\includegraphics[width=1.0\textwidth]{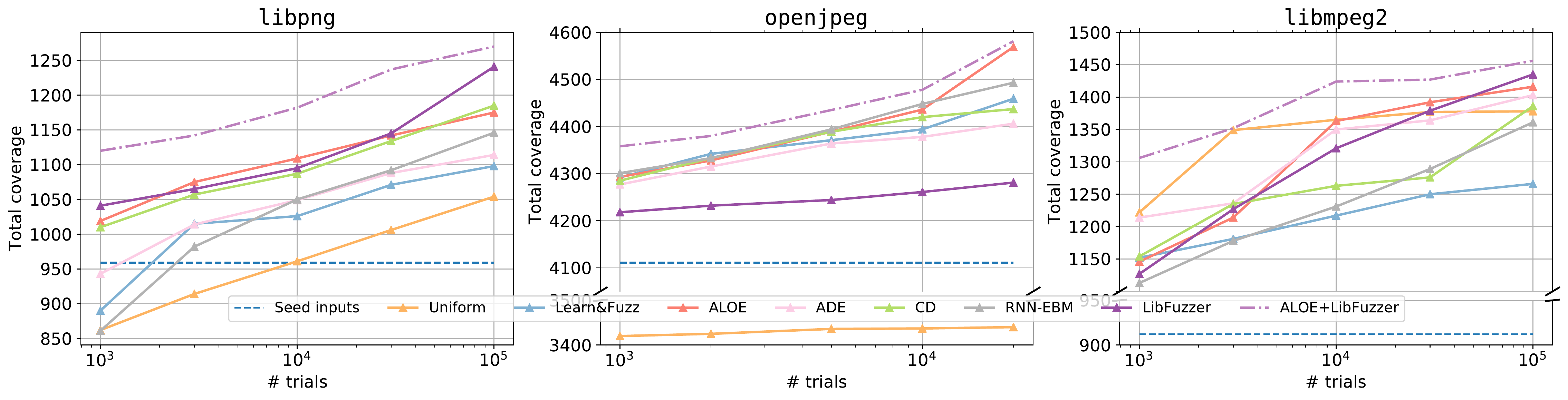}
\caption{Coverage statistics on different softwares with different amount of test inputs generated. \label{fig:fuzzing_curve} }
\end{figure*}

A critical step in software quality assurance is to generate random inputs to test the software for vulnerabilities, also known as fuzzing~\cite{miller1990empirical}. Recently learning based approaches have shown promising results for fuzzing~\citep{godefroid2017learn, dai2019learning}. In this section we focus on the generative fuzzing task, where we first learn a generative model from existing seed inputs (a set of software-dependent binary files) and then generate new inputs from the model to test the software. 
 
\noindent\textbf{Experiment setup:} We collect three software binaries (namely~\texttt{libpng},~\texttt{libmpeg2} and~\texttt{openjpeg}) from OSS-Fuzz\footnote{\url{https://github.com/google/oss-fuzz}} as test target. For all ML based methods, we use the seed inputs that come with OSS-Fuzz to learn the generative model. 
As each test input for software can be very large (\eg, a media file for~\texttt{libmpeg2}), we train a truncated EBM with a window size of 64. Specifically, we learn a conditional EBM $f(x|y)$, where $x \in \{0, \ldots, 255\}^{64}$ is a chunk of byte data and $y \in \{0, 1, \ldots\}$ is the position of this chunk in its original file. 

During inference, instead of generating test inputs from scratch (which would be too difficult to generate 1M bytes while still being parsable by the target software), we use the learned model to modify the seed inputs instead. To modify $i$-th byte of the byte string $\xb$ using learned EBM, we sample the byte $b \propto \exp( f([\xb_{i-31}, \ldots, b, \ldots, \xb_{i+32}]|i))$ by conditioning on its surrounding context.

We compare against the following generative model based methods:
\begin{itemize}[leftmargin=*,nosep]
	\item \texttt{Learn\&Fuzz}~\citep{godefroid2017learn}: this method learns an autoregressive model from sequences of byte data. We adapt its open-source implementation\footnote{\url{https://github.com/google/clusterfuzz/tree/master/src/python/bot/fuzzers/ml/rnn}}. To use the autoregressive model for mutating the seed inputs, we perform the edit by sampling $x_i \sim p(\cdot|\xb_{0:i-1})$ conditioned on its prefix.
	\item \texttt{ADE}~\citep{dai2019exponential}: This method parameterizes the model and initial sampler $q_0$ in the same way as \modelshort{}.
	\item \texttt{CD}: As PCD is not directly applicable for conditional EBM learning, we use CD instead. 
	\item \texttt{RNN-EBM}: It treats the autoregressive model learned by \texttt{Learn\&Fuzz} as an EBM, and mutates the seed inputs in the same way as other EBM based mutations.
\end{itemize}
We also use  uniform sampling (denoted as \texttt{Uniform}) over byte modifications as a baseline, and include \texttt{LibFuzzer} coverage with the same seed inputs as reference. Note that \texttt{LibFuzzer} is a well engineered system used commercially for fuzzing, which gathers feedback from the test program by monitoring which branches are taken during execution.
Therefore, this is supposed to be superior to generative approaches like EBMs, which do not incorporate this feedback. 

For all methods, we generate up to 100k inputs with 100 modifications for each.
The main evaluation measure is \emph{coverage}, which measures how many of the lines of code, branches, and so on, are exercised by the test inputs; higher is better. This statistic is reported by \texttt{LibFuzzer}\footnote{\url{https://llvm.org/docs/LibFuzzer.html}}. 

\noindent\textbf{Results} are shown in \figref{fig:fuzzing_curve}. 
Overall the discrete EBM learned by ALOE consistently outperforms the autoregressive model. 
Suprisingly, the coverage obtained by ALOE is comparable or even better than \texttt{LibFuzzer} on some targets, despite
the fact that \texttt{LibFuzzer} has access to more information about the program execution. In the long run, we believe
that this additional information will allow \texttt{LibFuzzer} to perform the best, it is still  appealing that ALOE has high sample efficiency
initially. Regarding several EBM based methods, we can see CD is comparable on \texttt{libpng} but for large target like \texttt{openjpeg} it performs much worse. ADE performs good initially on some targets but gets worse in the long run. Our hypothesis is that it is due to the lack of diversity, which suggests a potential mode drop problem that is common in REINFORCE based approaches.
The uniform baseline performs worst in most cases, except on \texttt{libmpeg2} early stage.
Our hypothesis is that the uniform fuzzer quickly triggers many branches that raise formatting errors, which explains its high coverage initially. 

% Our hypothesis is that it quickly finds the shallow coverage like format checking in \texttt{libmpeg2}. %while it is not able to explore effectively with more samples. 

We also combine the test inputs generated by \texttt{LibFuzzer} and ALOE (the orange dotted curve, for which the $x$-axis shows
the number of samples from each method). The coverage of this combined set of inputs is better than either individually,
showing that the methods are complementary.

\subsection{Program Synthesis}
\label{sec:exp_robustfill}

In program synthesis, the task is to predict the source code of a program given a few input-output (IO) pairs
that specify its behavior.
We evaluate ALOE on the RobustFill task~\citep{devlin2017robustfill} of generating string transformations.The purpose here is to evaluate the effect of proposed local edits in our sampler, so the other methods like ADE or PCD are not applicable here. 
For full details, see~\appref{app:exp_robustfill}. 

\begin{figure}[t]
\begin{minipage}{0.67\textwidth}
\captionof{table}{Program synthesis accuracy on RobustFill tasks~\citep{devlin2017robustfill}. \label{tab:robustfill_acc} }
\resizebox{1.0\textwidth}{!}{%
\begin{tabular}{cccc}
	\toprule
	& Top-1 Beam-1 & Top-1 Beam-10 & Top-10 Beam-10 \\
	\hline
	seq2seq-init & 45.86 & 55.49 & 58.66 \\
	seq2seq-tune & 47.86 & 57.52 & 60.62 \\
	\modelshort{} & {\bf 53.57} & {\bf 61.99} & {\bf 65.29} \\
	\bottomrule
\end{tabular}
}
\vspace{-2mm}
\end{minipage}%
\hfill
\begin{minipage}{0.3\textwidth}
	\includegraphics[width=1.0\textwidth]{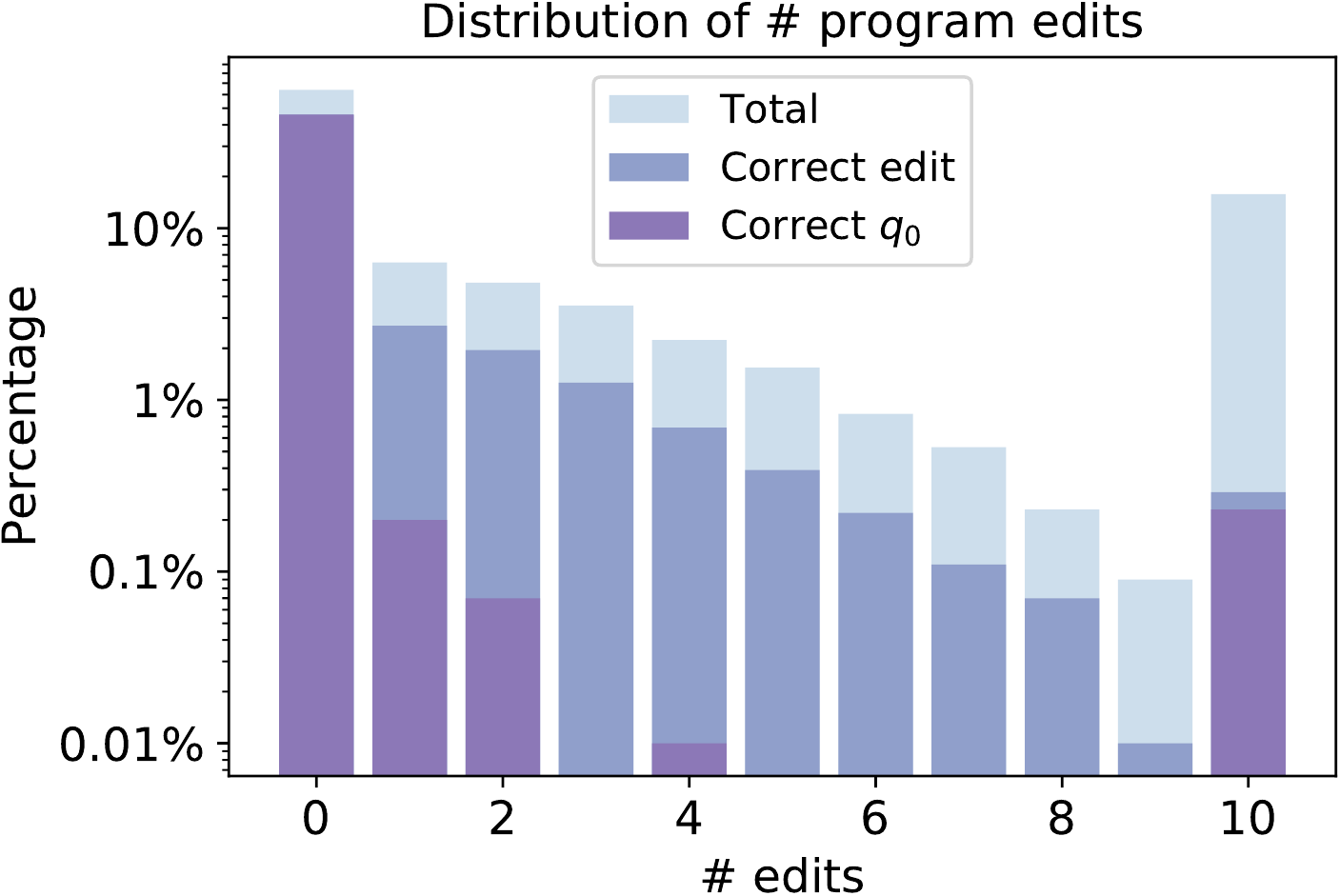}
%	\caption{Distribution of \# edits. \label{fig:n_edit_robustfill}}
\end{minipage}
\vspace{-2mm}
\end{figure}

\noindent \textbf{Experiment setup:} Data is generated synthetically, following~\citet{devlin2017robustfill}.
Each example in the data set is a synthesis task where 
the input is four IO pairs, the target is a program, and a further six IO pairs are held out for evaluation.
The training data is generated on the fly, while we keep 10k test examples for evaluation. Each target program consists of at most 10 sub-expressions in a domain-specific languages which includes string concatenation, substring operations, etc.
We report accuracy, which measures when the predicted program is consistent with all  $10$ IO pairs.

For ALOE we learn a conditional sampler $q(x|z)$ where $x$ is the program syntax tree, and $z$ is the list of input-output pairs. We compare with 3-layer seq2seq model for program prediction. Both seq2seq and ALOE share the same IO-pair encoder. As mentioned in \secref{sec:learn_local_search}, the initial distribution $q_0$ is the same as seq2seq autoregressive model, while subsequent modifications $A(x_{i+1}|x_i)$ adds/deletes/replaces one of the subexpressions. 
We train baseline seq2seq with 483M examples (denoted as seq2seq-init), and fine-tune with additional 264M samples (denoted as seq2seq-tune) with reduced learning rate. ALOE initializes $q_0$ from seq2seq-init and set it to be fixed, and train the editor $q_A(\cdot|\cdot)$ with same additional number of samples with the shortest edit importance proposal~\eqref{eq:short_proposal}.

\noindent\textbf{Results:}
We report the top-$k$ accuracy with different beam-search sizes in~\tabref{tab:robustfill_acc}. We can see ALOE outperforms the seq2seq baseline by a large margin. Although the initial sampler $q_0$ is the same as seq2seq-init, the editor $q_A(\cdot|\cdot)$ is able to further locate and correct sub-expressions of the initial prediction. In the figure to the right of~\tabref{tab:robustfill_acc}, We also visualize the number of edits our sampler makes on the test set. In most cases $q_0$ already produces correct results, and the sampler correctly learns to stop at step $0$. From 1 to 9 edits we can see the editor indeed improved from $q_0$ by a large margin. 
There are many cases which require $10$ or more edits, in which case we truncate the local search steps to 10. Some of them are difficult cases where the sampler learns to ask for more steps, while for others the sampler keeps modifying to semantically equivalent programs.

\vspace{-2mm}
\section{Conclusion}
\label{sec:conclusion}
\vspace{-2mm}

In this paper, we propose ALOE, a new algorithm for learning discrete EBMs for both conditional and unconditional cases. ALOE learns a sampler that is parameterized as a local search algorithm for proposing negative samples in contrastive learning framework. With an efficient importance reweighted gradient estimator, we are able to train both the sampler and the EBM with a variational power iteration principle. 
Experiments on both synthetic datasets and real-world software testing and program synthesis tasks show that both the learned EBM and local search sampler outperforms the autoregressive alternative. Future work includes better approximation of learning local search algorithms, as well as extending it to other discrete domains like chemical engineering and NLP.

\clearpage
\newpage

\section*{Broader Impact}

We hope our new algorithm ALOE for learning discrete EBMs can be useful for different domains with discrete structures, and it furthers the general research efforts in this direction of generative models of discrete structures. In this paper, we present its application to program synthesis and software fuzzing. A positive outcome of improved performance in program synthesis would be that it can help democratize the task of programming by allowing people to express their desired intent using input-output examples without the need of learning complex programming languages. Similarly, a positive outcome of improvements in software fuzzing could allow software developers to identify bugs and vulnerabilities quicker and in turn improve software reliability and robustness.

A possible negative outcome could be that malicious attackers might also use such technology to discover software vulnerabilities and use it for undesirable purposes~\cite{brundage2018malicious}. However, this outcome is not specific to our technique but more generally applicable to the large research field of software fuzzing, and there is a large amount of work in the fuzzing field for accounting ethical considerations. For example, the vulnerabilities typically found by fuzzers is first responsibly disclosed to corresponding software teams~\citep{householder2017cert} that gives them enough time to patch the vulnerabilities before the bugs and vulnerabilities are released publicly. 

\begin{ack}
We would like to thank Sherry Yang for helping with fuzzing experiments. 
We would also like to thank Adams Wei Yu, George Tucker, Yingtao Tian and anonymous reviewers for valuable comments and suggestions. 
\end{ack}

\clearpage
\newpage

\appendix

\begin{appendix}

\numberwithin{equation}{section}
\numberwithin{table}{section}
\numberwithin{figure}{section}

\thispagestyle{plain}
\begin{center}
{\huge Appendix}
\end{center}

\section{More experiment details}

\subsection{Synthetic experiments}
\label{app:exp_synthetic}

\paragraph{Baseline configuration}

We adapt two most recent techniques for learning EBMs into discrete case, namely~\citet{du2019implicit} and~\citet{dai2019exponential}. Specifically:

\begin{itemize}[leftmargin=*,nolistsep,nosep]
	\item \textbf{PCD based}: ~\citet{du2019implicit} extends the PCD method with replay buffer and random restart. We adapt these tricks in learning discrete EBMs. Specifically, we use Gibbs sampling for $K\times 32$-steps as the MCMC sampler, where $K$ is set to 10. Instead of always inheriting from previous MCMC samples, we tune the restart rate in $\{0.05, 0.1, 1\}$.
	\item \textbf{ADE based}: ADE solves the same minimax problem in Eq~\eqref{eq:minimax}, but instead directly minimizes the objective $L(q) := -\EE_{x \sim q}\sbr{f(x)} - H(q)$. To make ADE work in discrete case, optimizing $L(q)$ requires the policy gradient technique with variance reduction~\citep{mnih2014neural, gu2015muprop, tucker2017rebar, grathwohl2017backpropagation}, where the gradient estimator becomes $\nabla_q L(q) = \EE_{x \sim q} \nabla \log q(x) (-f(x) -\log q(x) - 1)$. This also resembles the learning of GAN~\citep{goodfellow2014generative} on sequences~\citep{yu2017seqgan} or graphs~\citep{de2018molgan}, except the additional entropy regularization term and constant. We use A2C~\citep{mnih2014neural} to learn ADE for discrete EBMs. As ADE uses alternating minimization for minimax problem, we tune the learning rate ratio and synchronization frequency between energy function and sampler learning in $\{0.2, 0.5, 1\}$ and $\{1:1, 1:3, 1:5\}$, respectively. 
\end{itemize}

\paragraph{More visualizations}

\begin{figure*}[h]
\centering
\setlength{\tabcolsep}{0.1em}
\begin{tabular}{ccccccc}
	\includegraphics[width=0.14\textwidth]{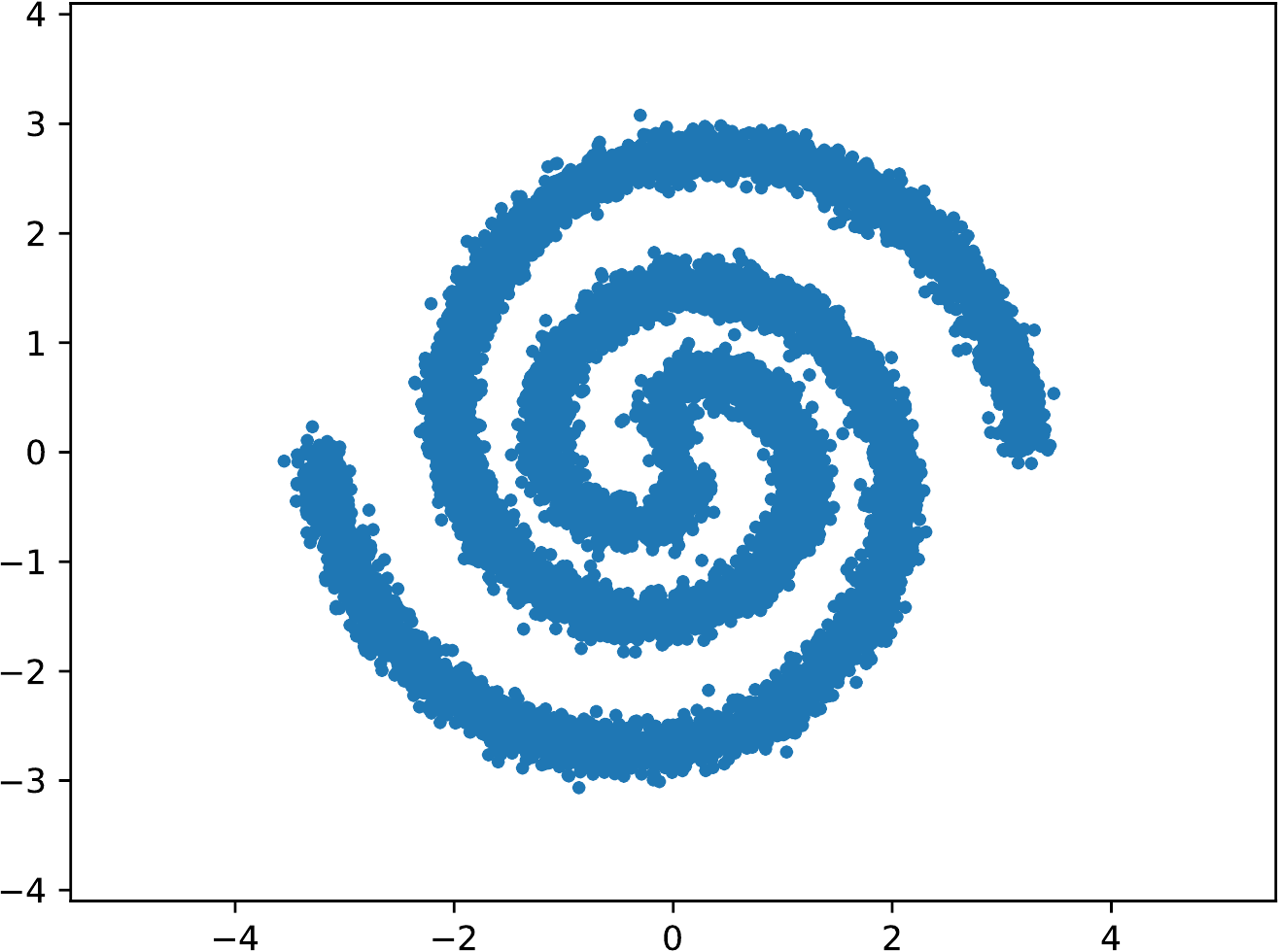} &
	\includegraphics[width=0.14\textwidth]{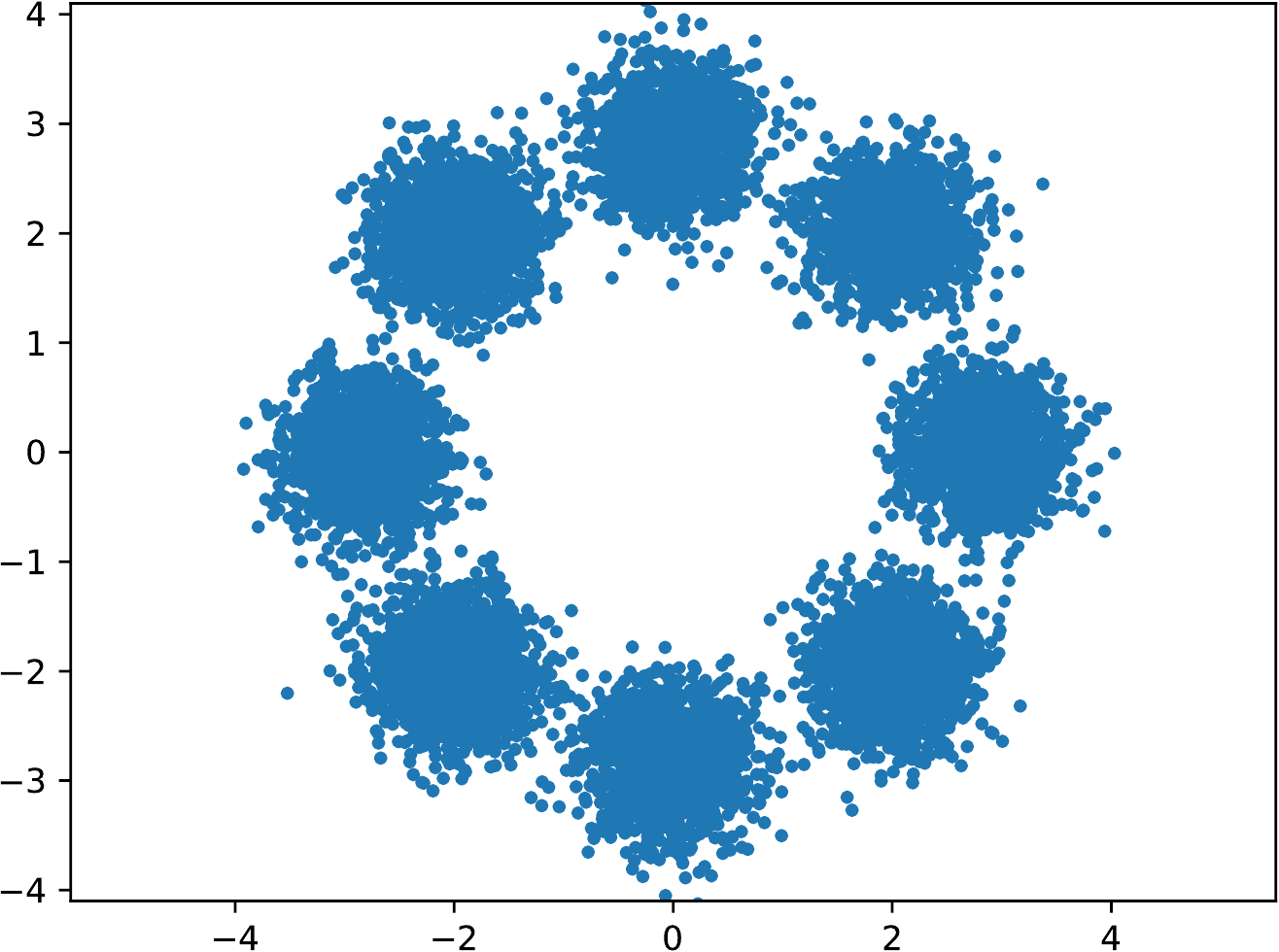} &
	\includegraphics[width=0.14\textwidth]{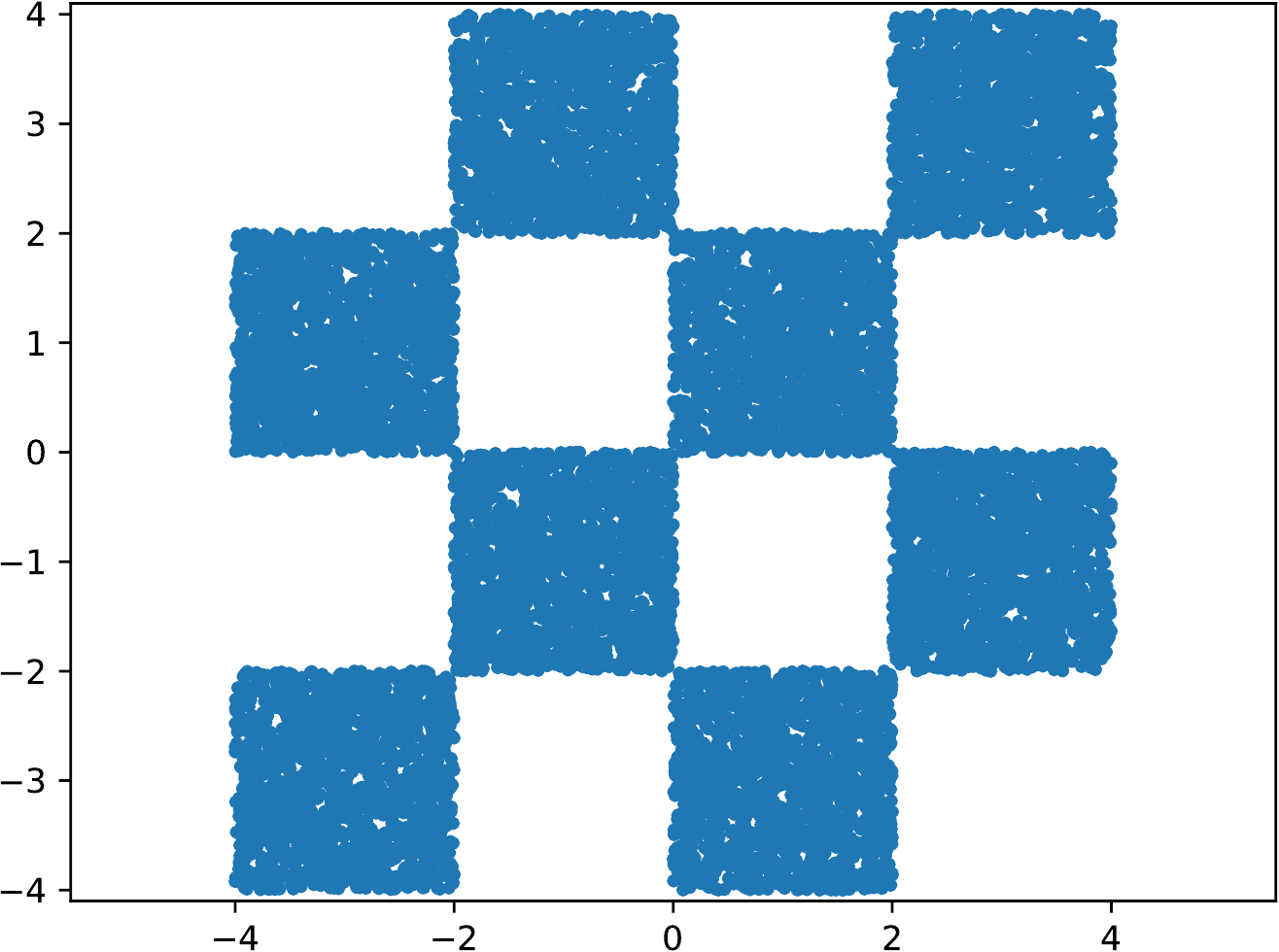} &
	\includegraphics[width=0.14\textwidth]{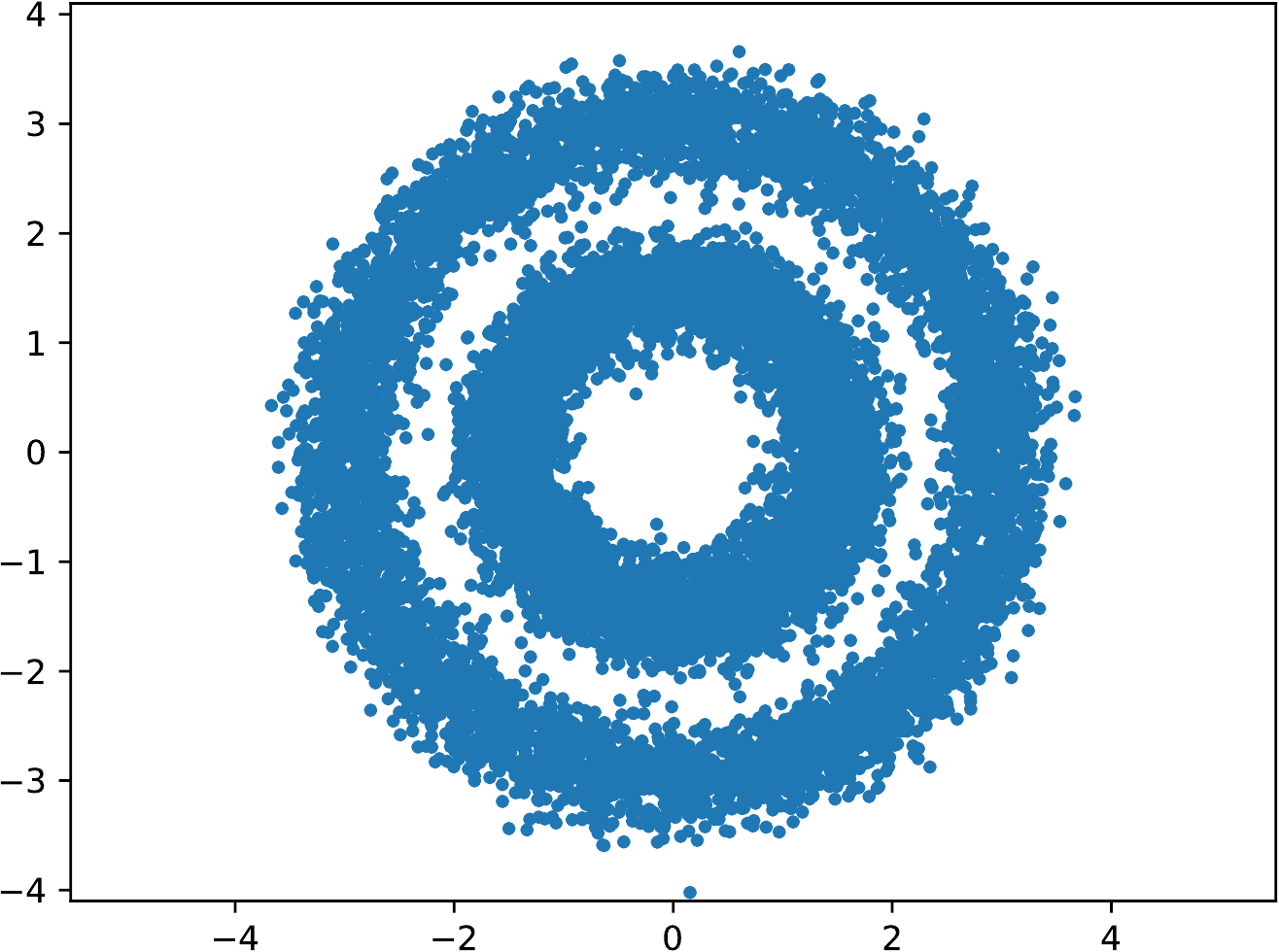} &
	\includegraphics[width=0.14\textwidth]{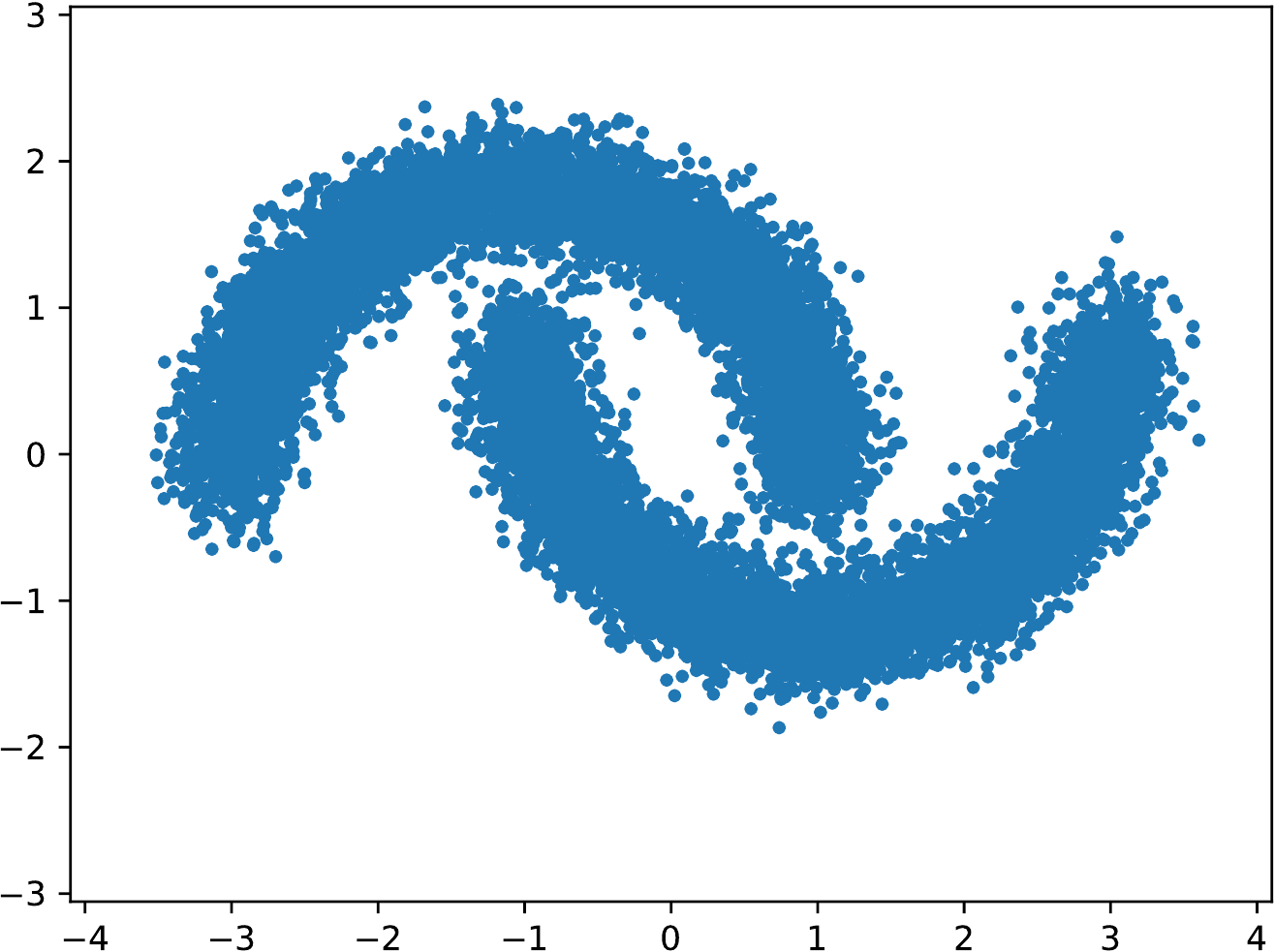} &
	\includegraphics[width=0.14\textwidth]{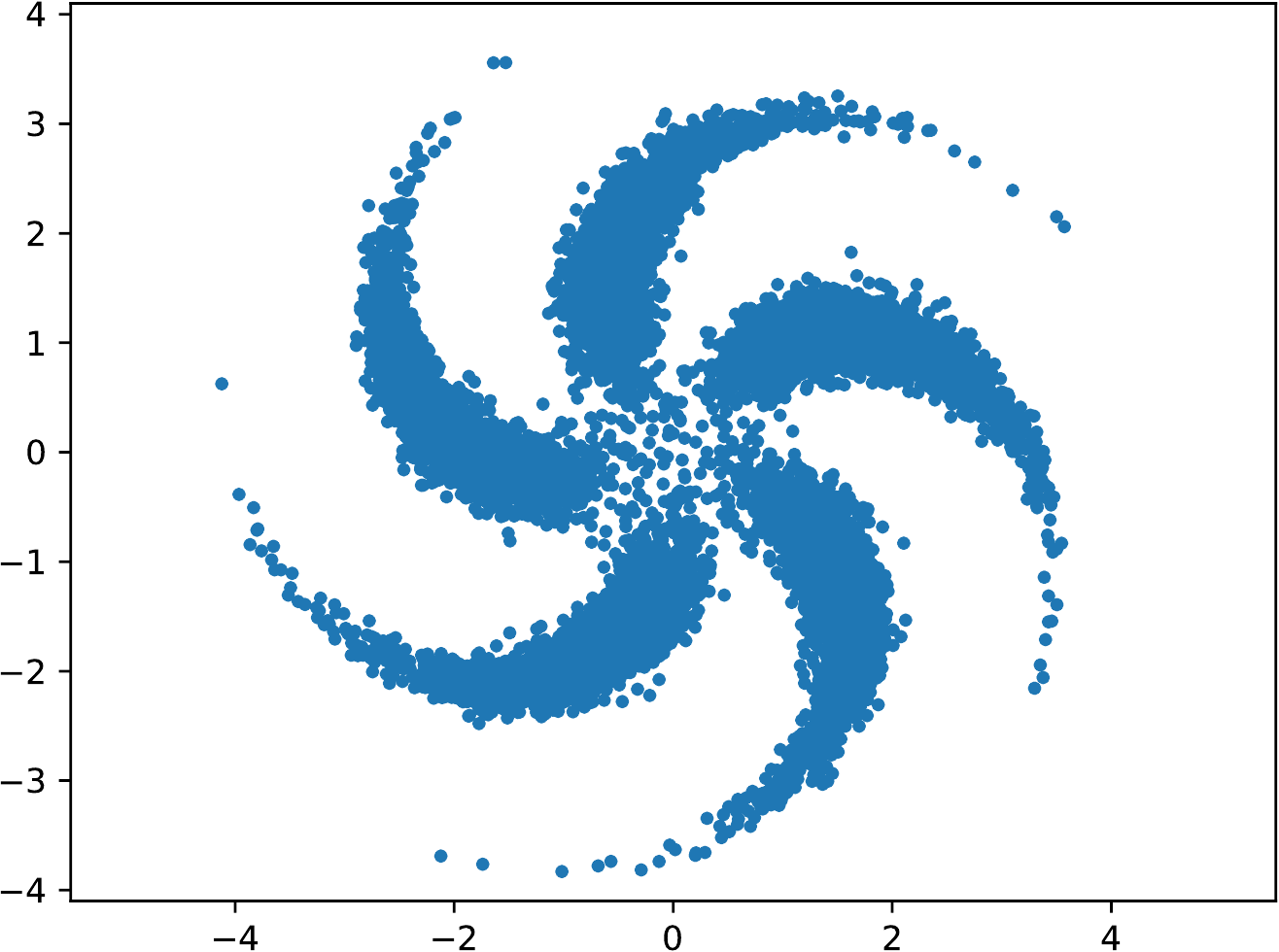} &
	\includegraphics[width=0.14\textwidth]{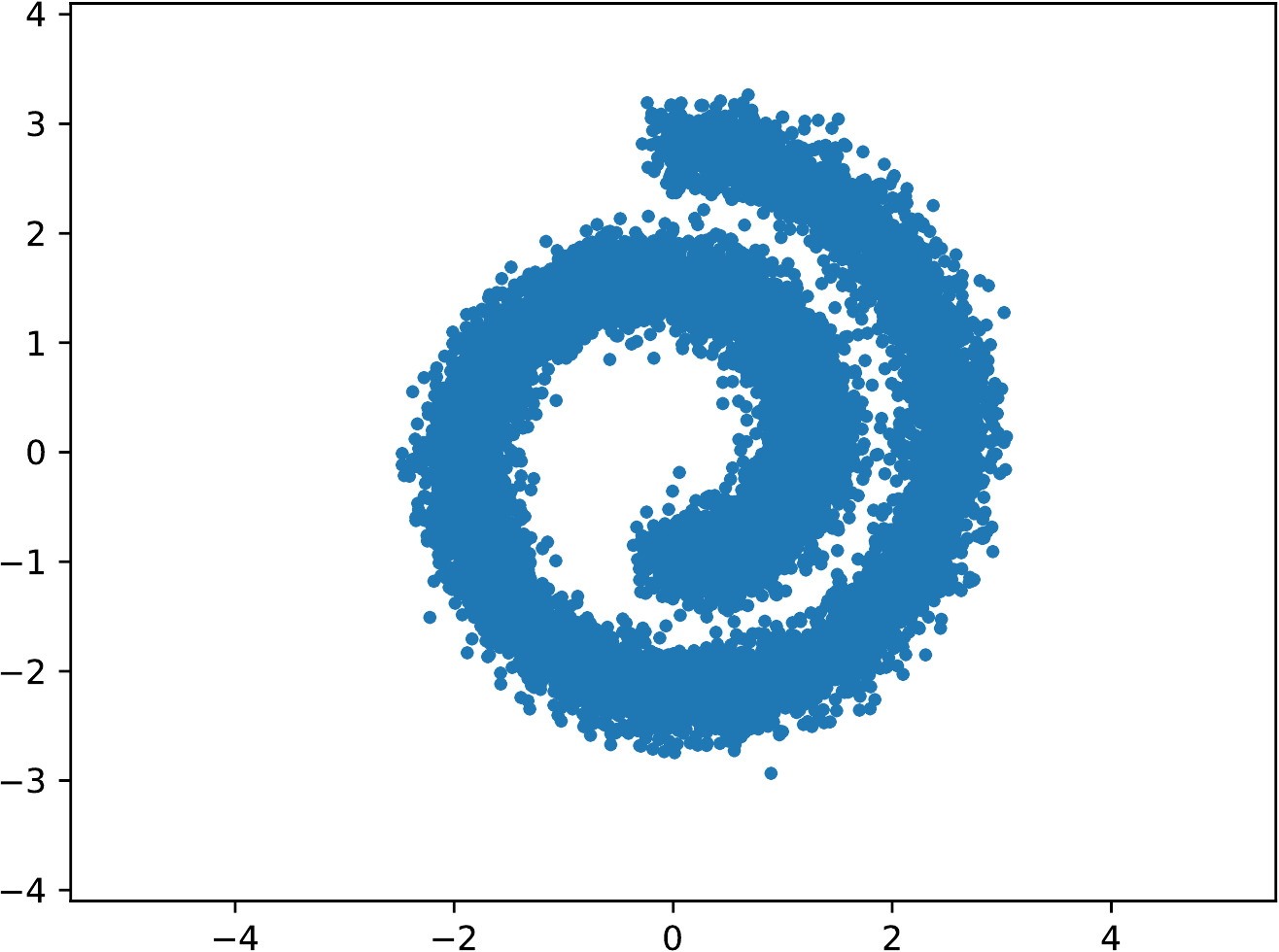} 
	\\
	2spirals & 8gaussians & checkerboard & circles & moons & pinwheel & swissroll \\
\end{tabular}
\caption{2D visualization of samples from the ground truth distribution. \label{fig:sample_gt}}
\end{figure*}

\begin{figure*}[h]
\centering
\setlength{\tabcolsep}{0.1em}
\begin{tabular}{cccccccc}
	\multicolumn{7}{c}{PCD*} \\
	\includegraphics[width=0.14\textwidth]{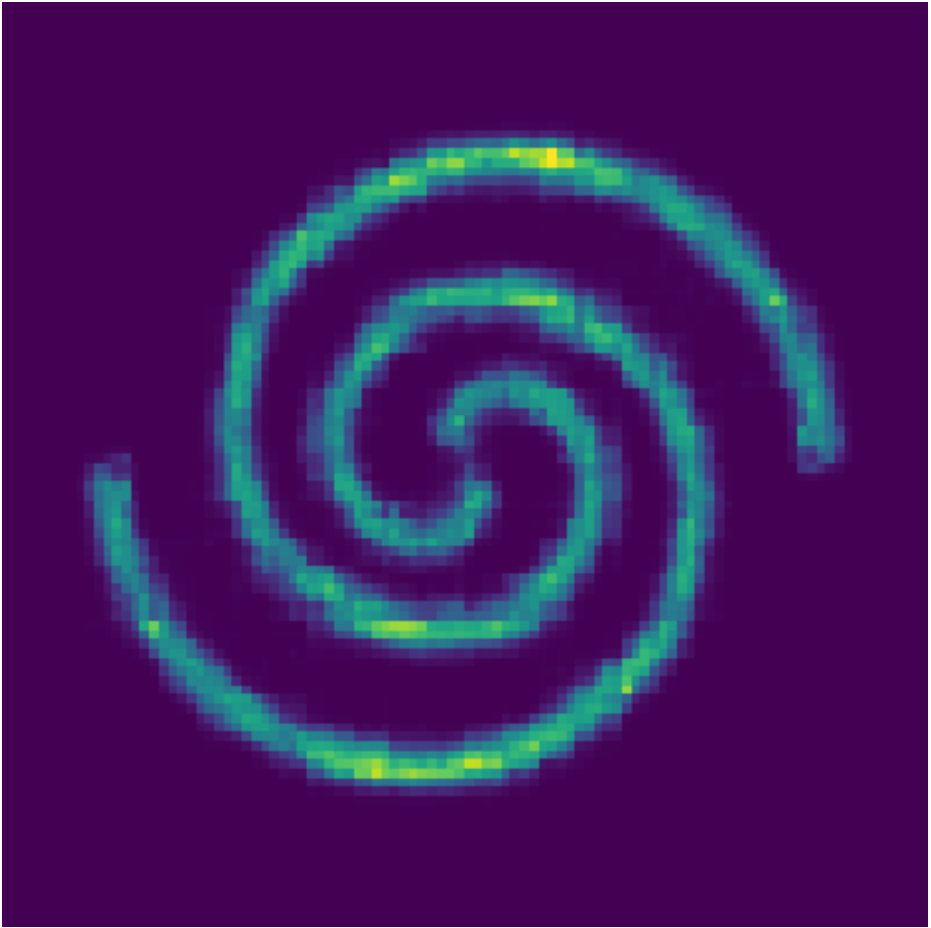} &
	\includegraphics[width=0.14\textwidth]{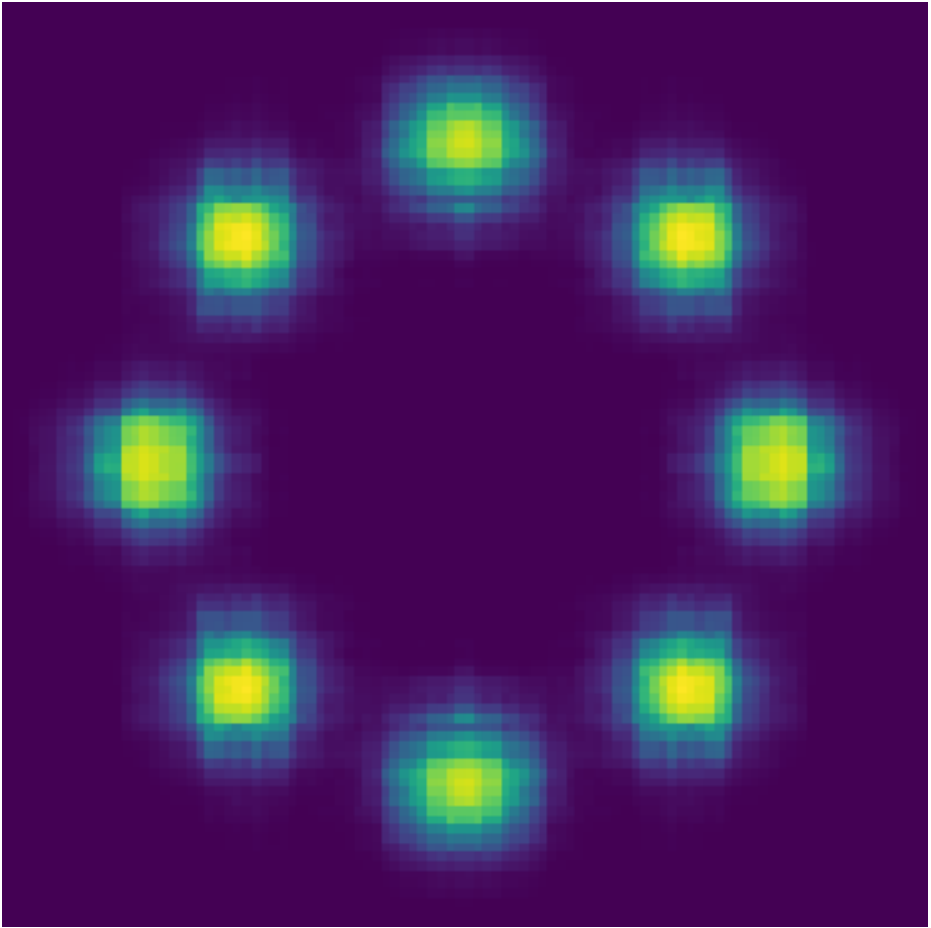} &
	\includegraphics[width=0.14\textwidth]{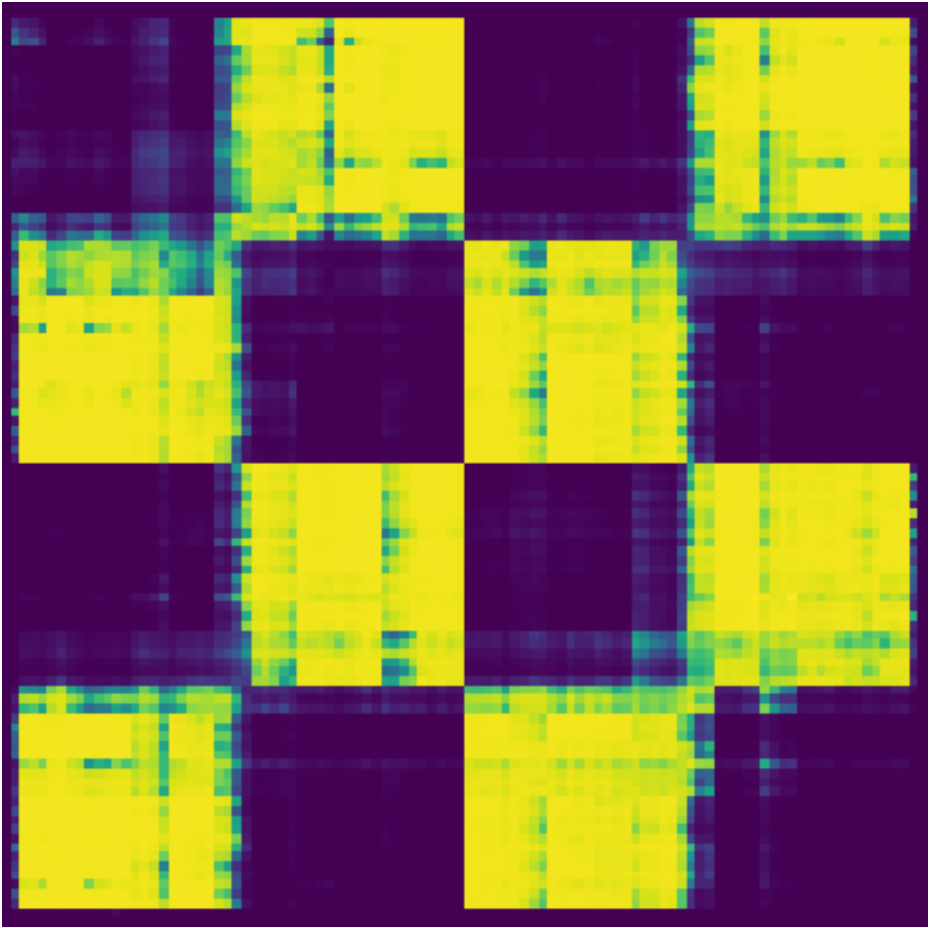} &
	\includegraphics[width=0.14\textwidth]{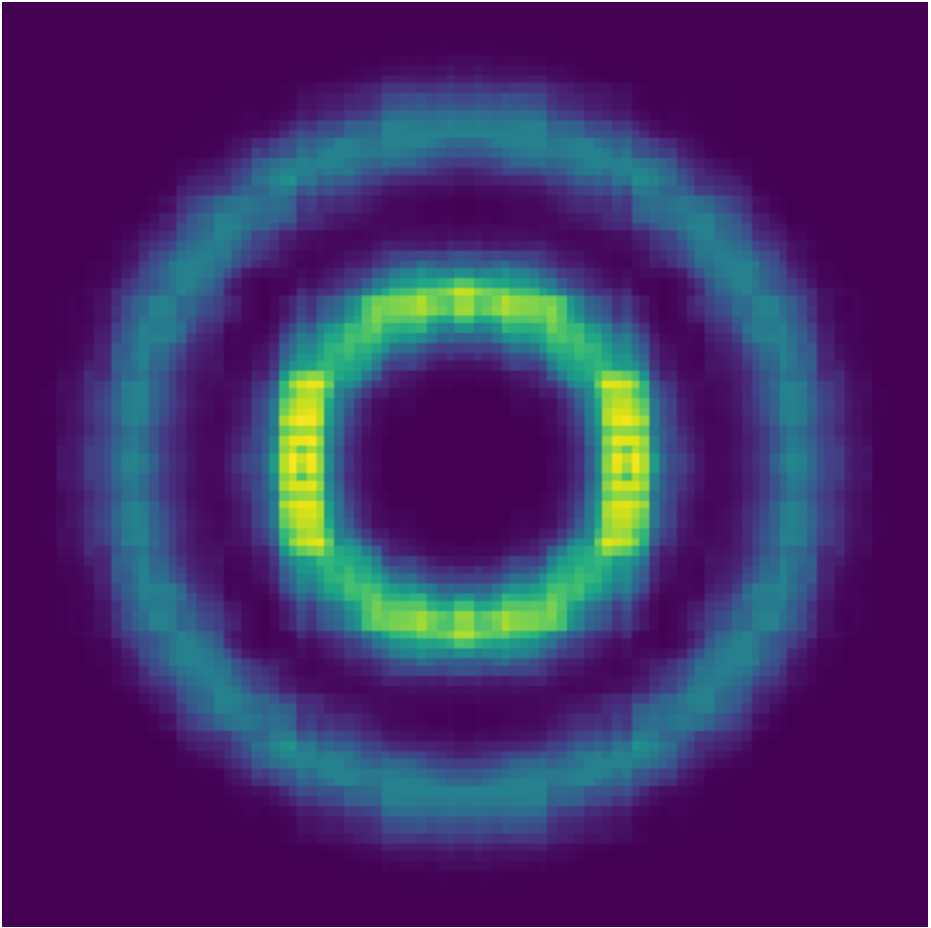} &
	\includegraphics[width=0.14\textwidth]{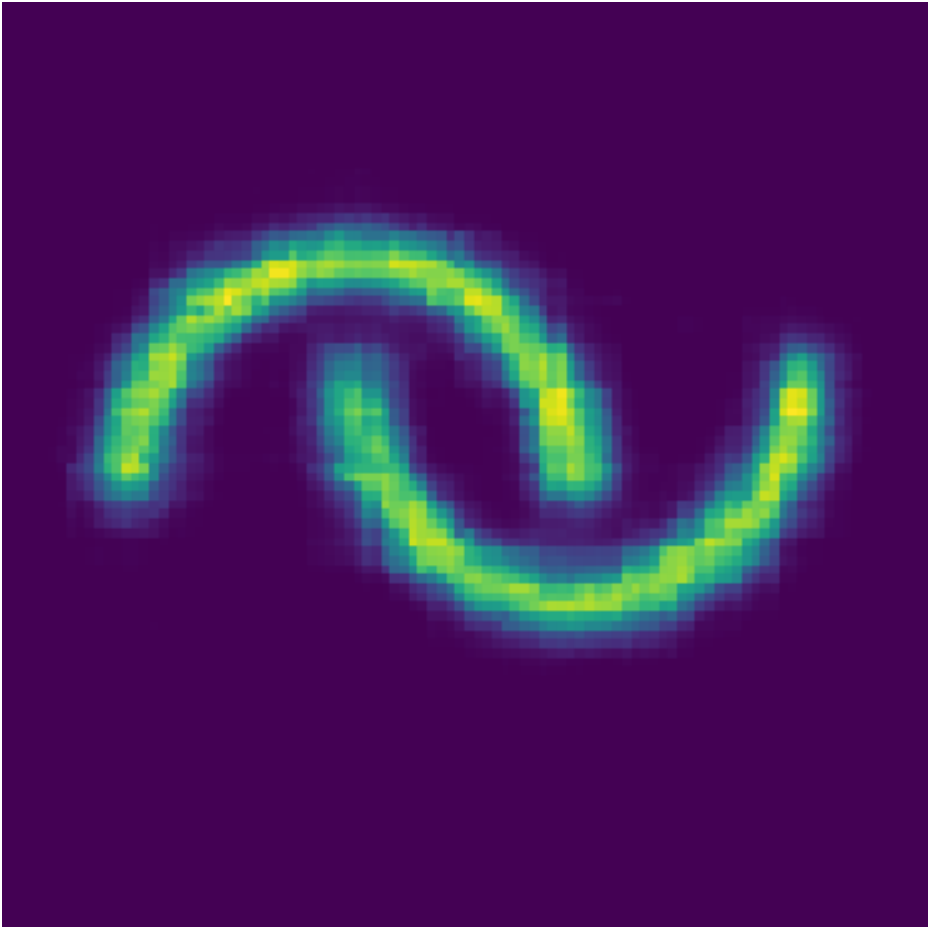} &
	\includegraphics[width=0.14\textwidth]{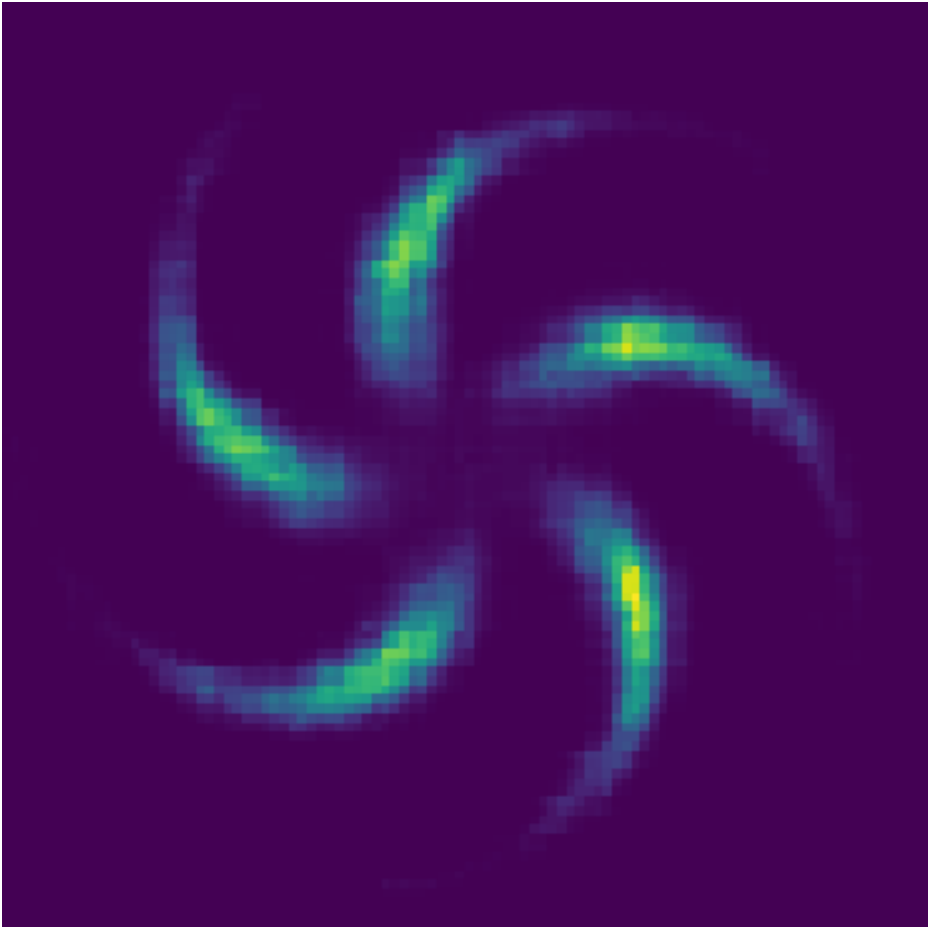} &
	\includegraphics[width=0.14\textwidth]{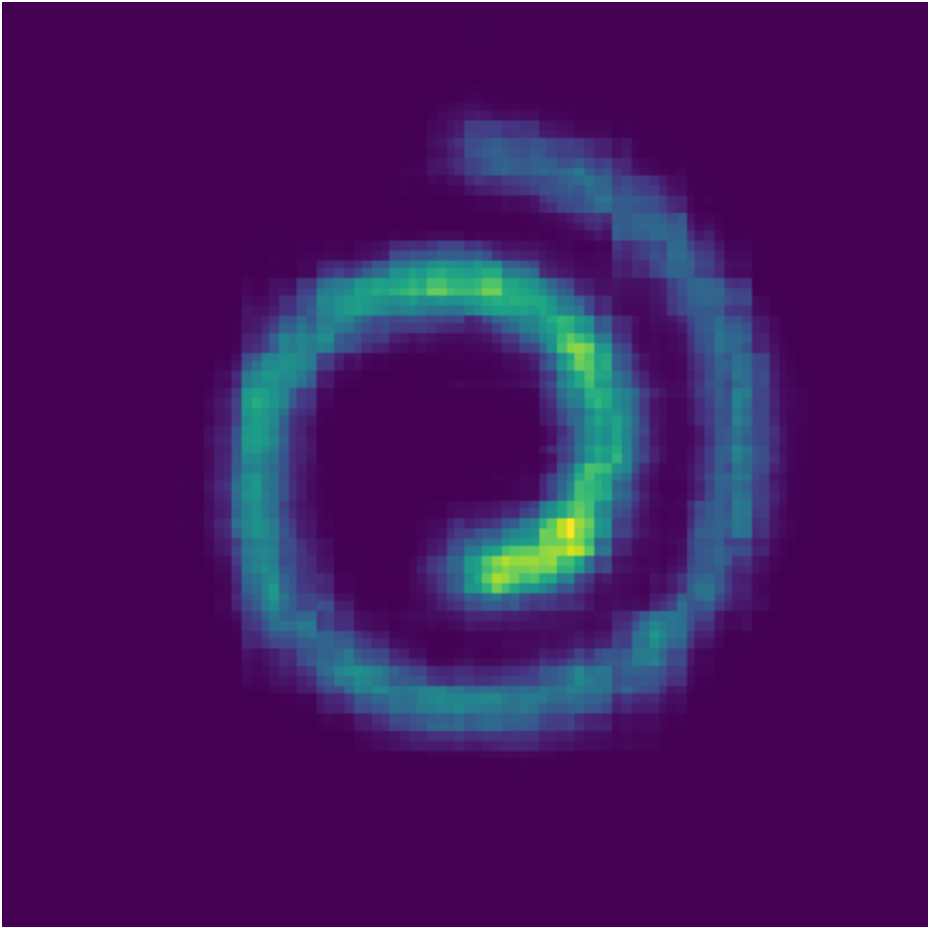} 
	\\
	\multicolumn{7}{c}{ADE*} \\
	\includegraphics[width=0.14\textwidth]{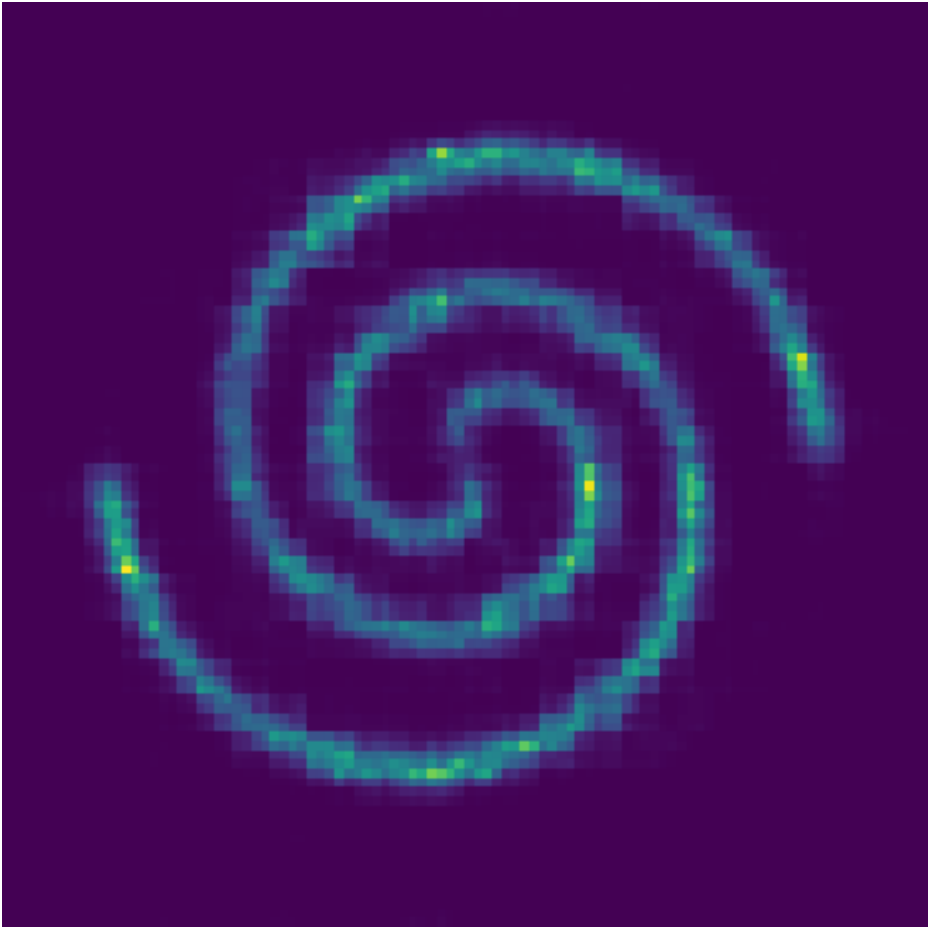} &
	\includegraphics[width=0.14\textwidth]{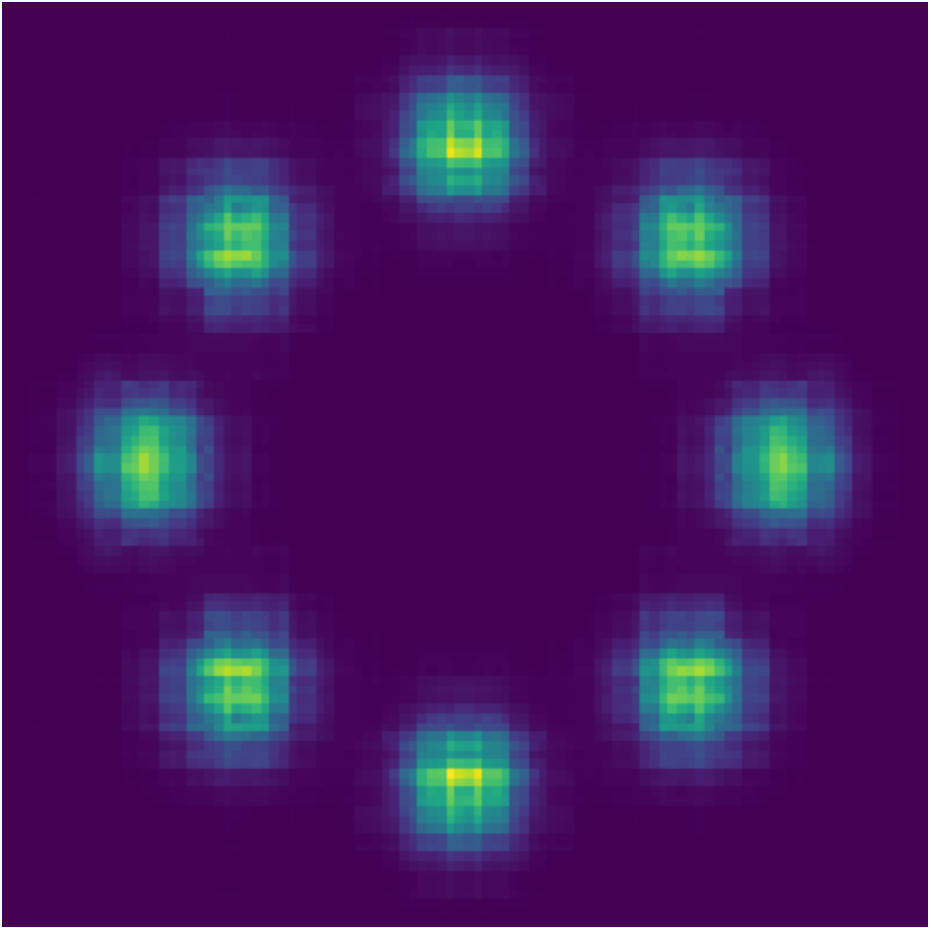} &
	\includegraphics[width=0.14\textwidth]{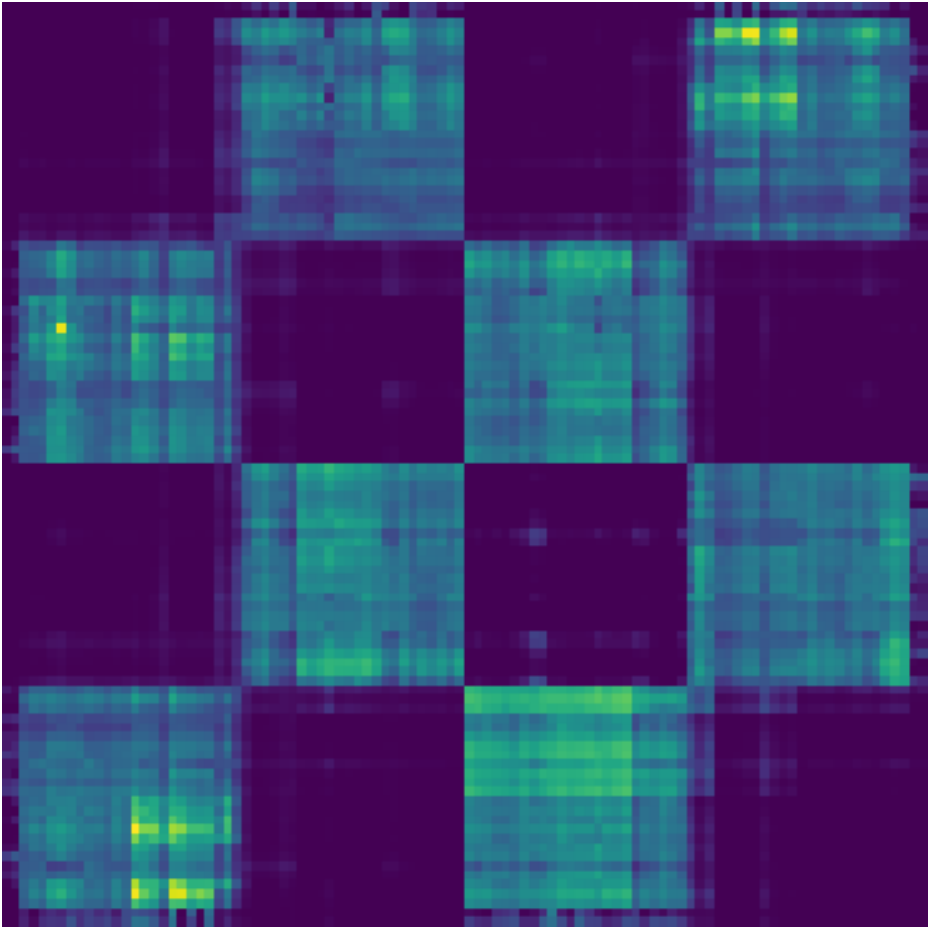} &
	\includegraphics[width=0.14\textwidth]{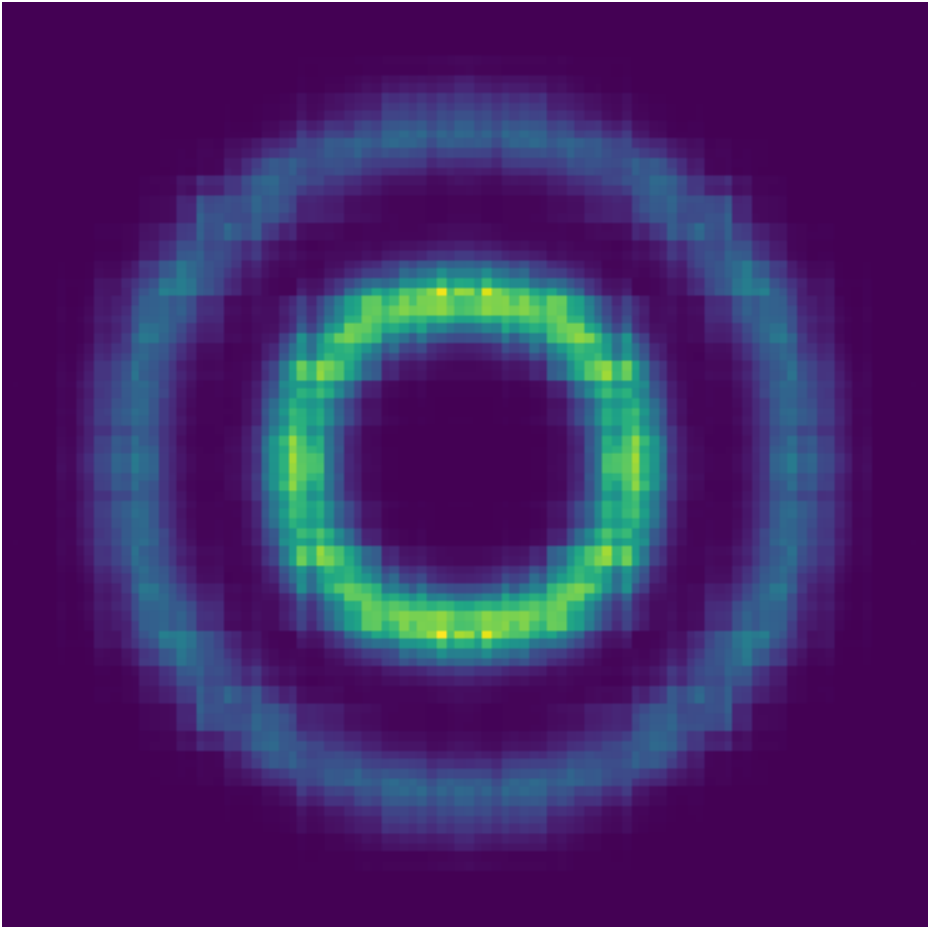} &
	\includegraphics[width=0.14\textwidth]{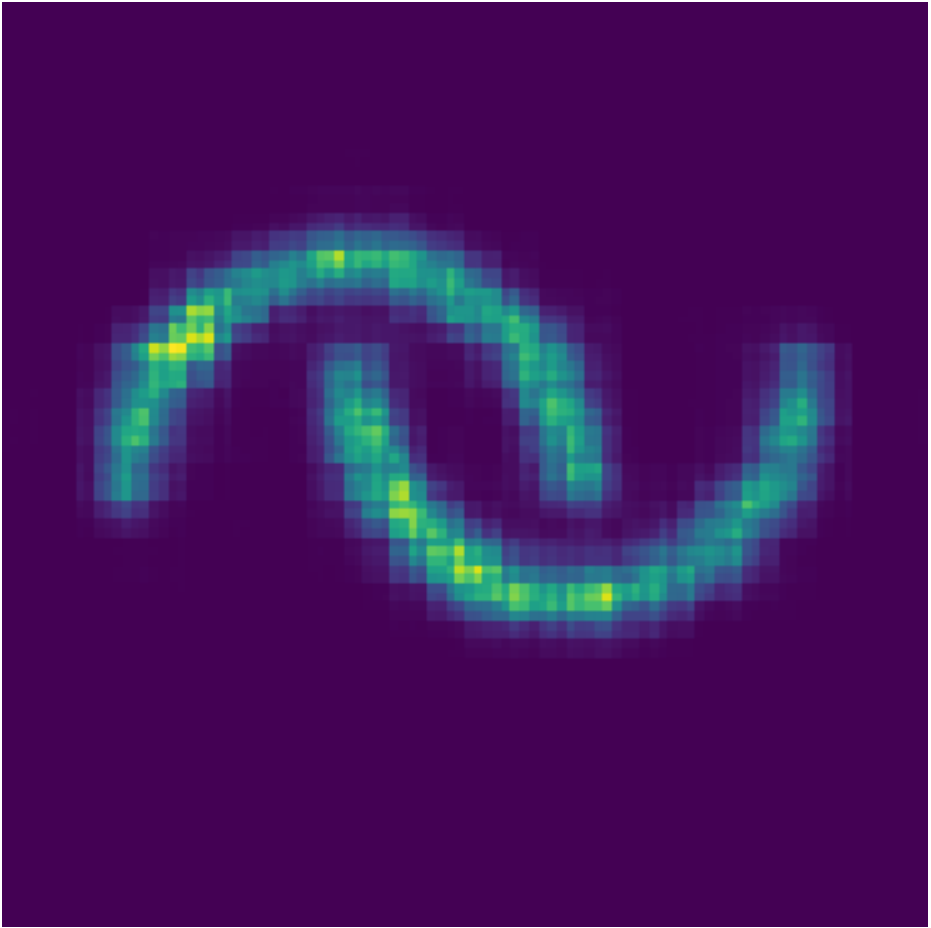} &
	\includegraphics[width=0.14\textwidth]{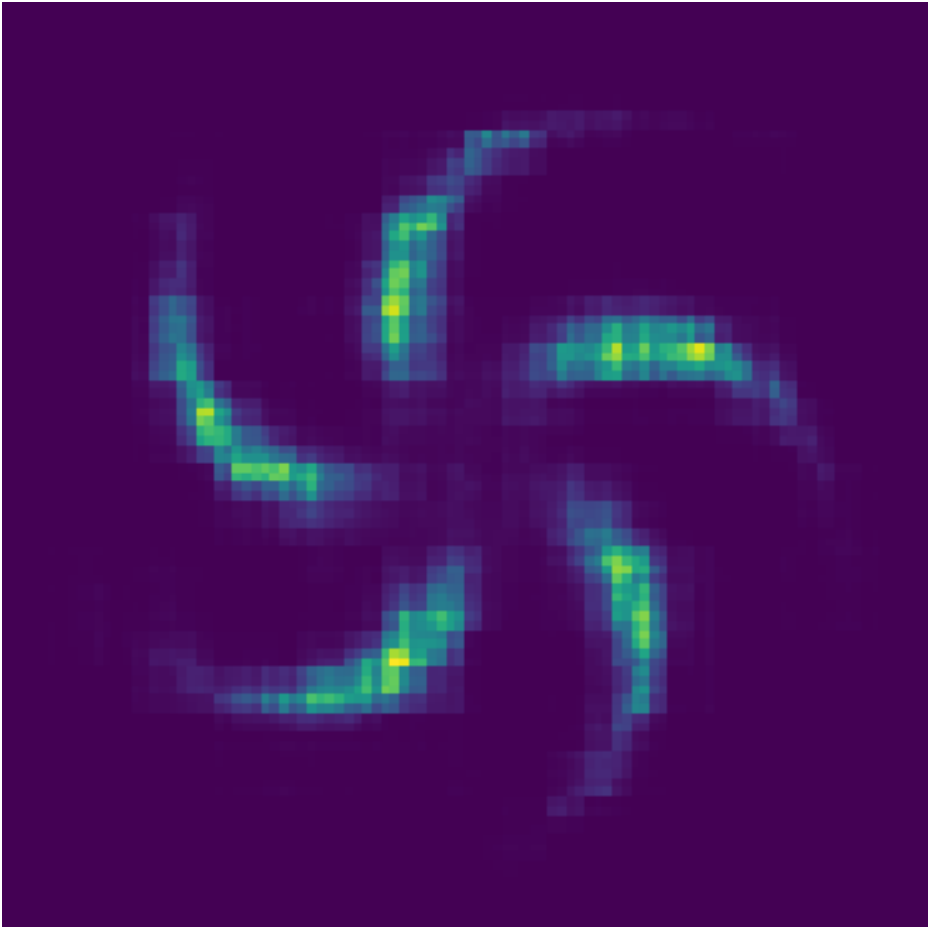} &
	\includegraphics[width=0.14\textwidth]{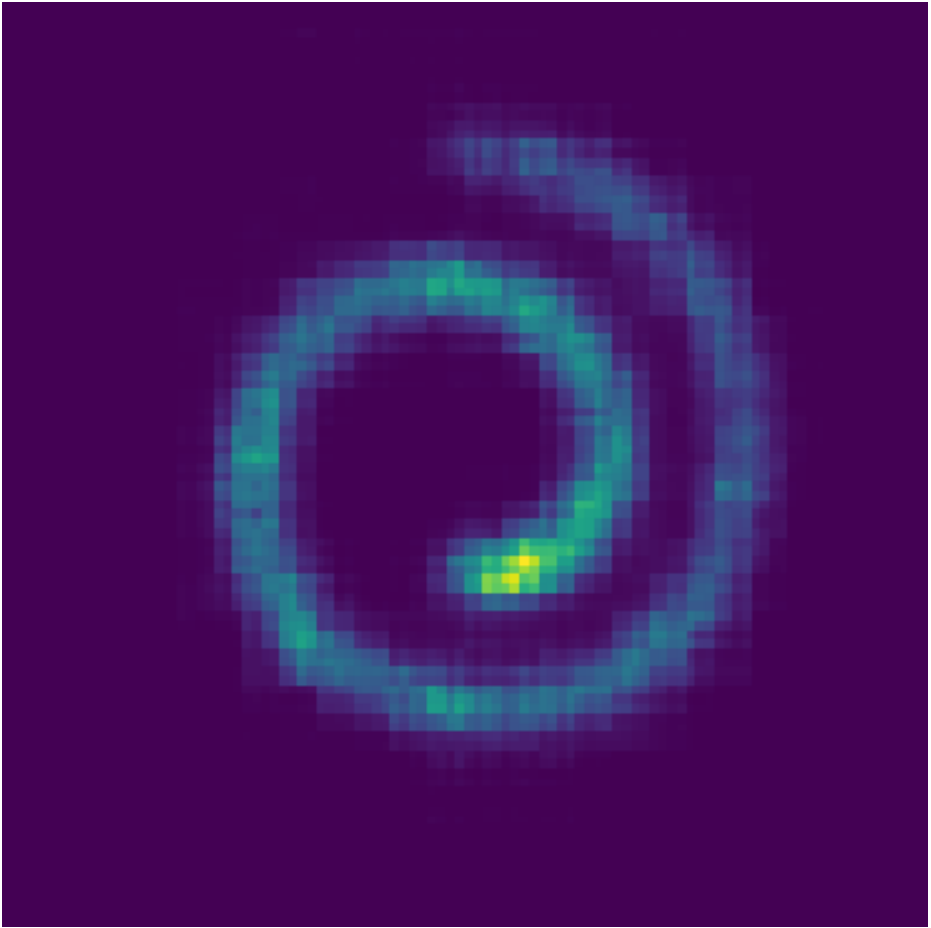}
	\\
	\multicolumn{7}{c}{\modelshort{}} \\
	\includegraphics[width=0.14\textwidth]{figs/2spirals-heat} & 
	\includegraphics[width=0.14\textwidth]{figs/8gaussians-heat} & 
	\includegraphics[width=0.14\textwidth]{figs/checkerboard-heat} & 
	\includegraphics[width=0.14\textwidth]{figs/circles-heat} & 
	\includegraphics[width=0.14\textwidth]{figs/moons-heat} & 
	\includegraphics[width=0.14\textwidth]{figs/pinwheel-heat} &
	\includegraphics[width=0.14\textwidth]{figs/swissroll-heat}
	\\
	2spirals & 8gaussians & checkerboard & circles & moons & pinwheel & swissroll \\
\end{tabular}
\caption{visualization of learned discrete EBMs using different methods. \label{fig:baseline_vis}}
\end{figure*}

In \figref{fig:sample_gt} we also visualize the samples obtained from the ground truth distribution and visualize them in 2D space. Compared to \figref{fig:synthetic_vis} we can see our learned sampler can almost perfectly recover the true distribution. The checkerboard seems to be the most difficult one among these datasets, as for both PCD and ADE baselines the learned model is much worse than the one learned by \modelshort{}. We find that in this case the distribution is not smooth as it has sharp boundaries for each ``square'' in the distribution. Thus below we study how the learned sampler behaves for ADE algorithm in this case. 
\begin{wrapfigure}{r}{0.43\textwidth}
\centering
\vspace{-2mm}
\begin{tabular}{cc}
	\includegraphics[width=0.2\textwidth]{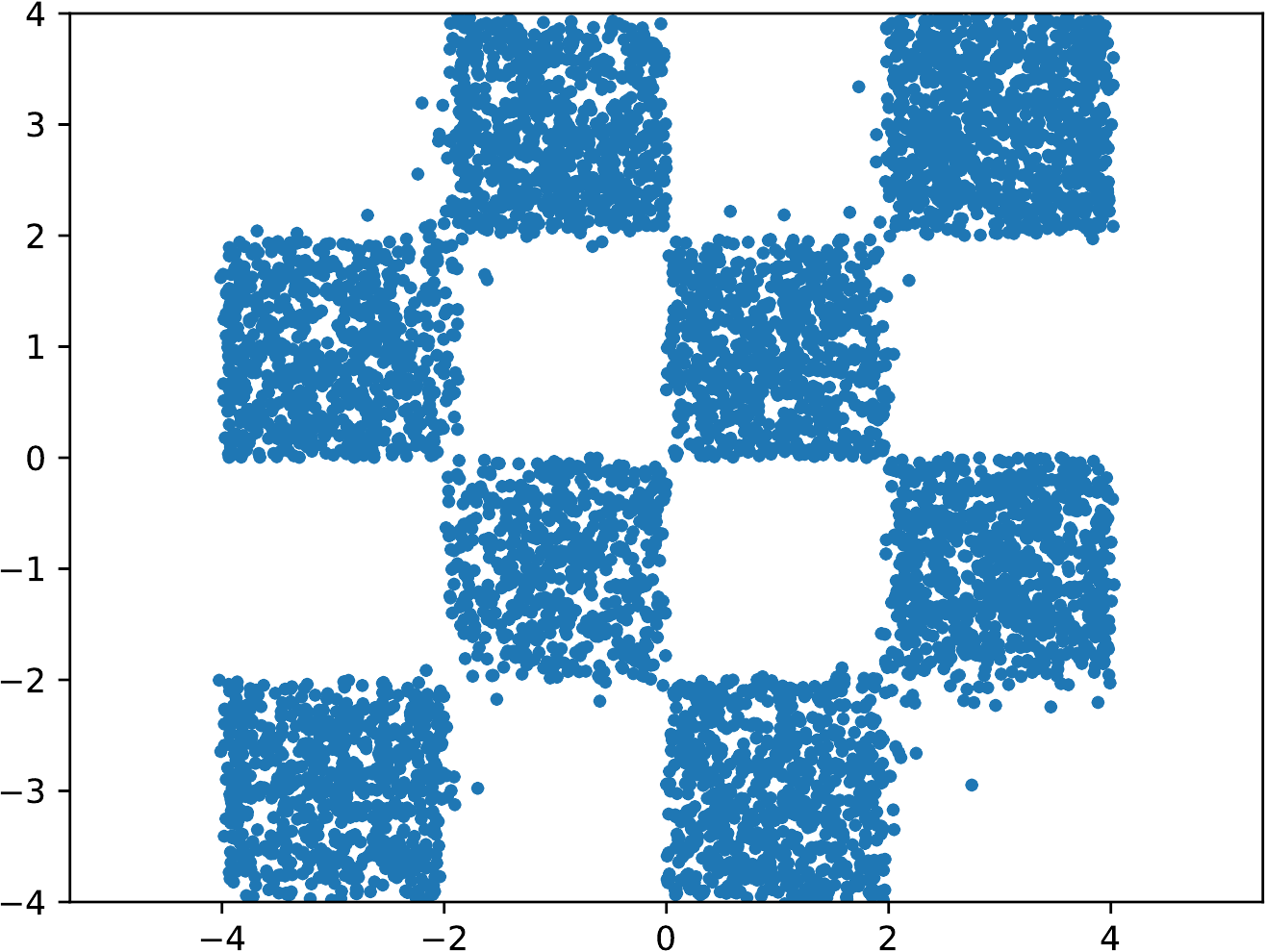} & 
	\includegraphics[width=0.2\textwidth]{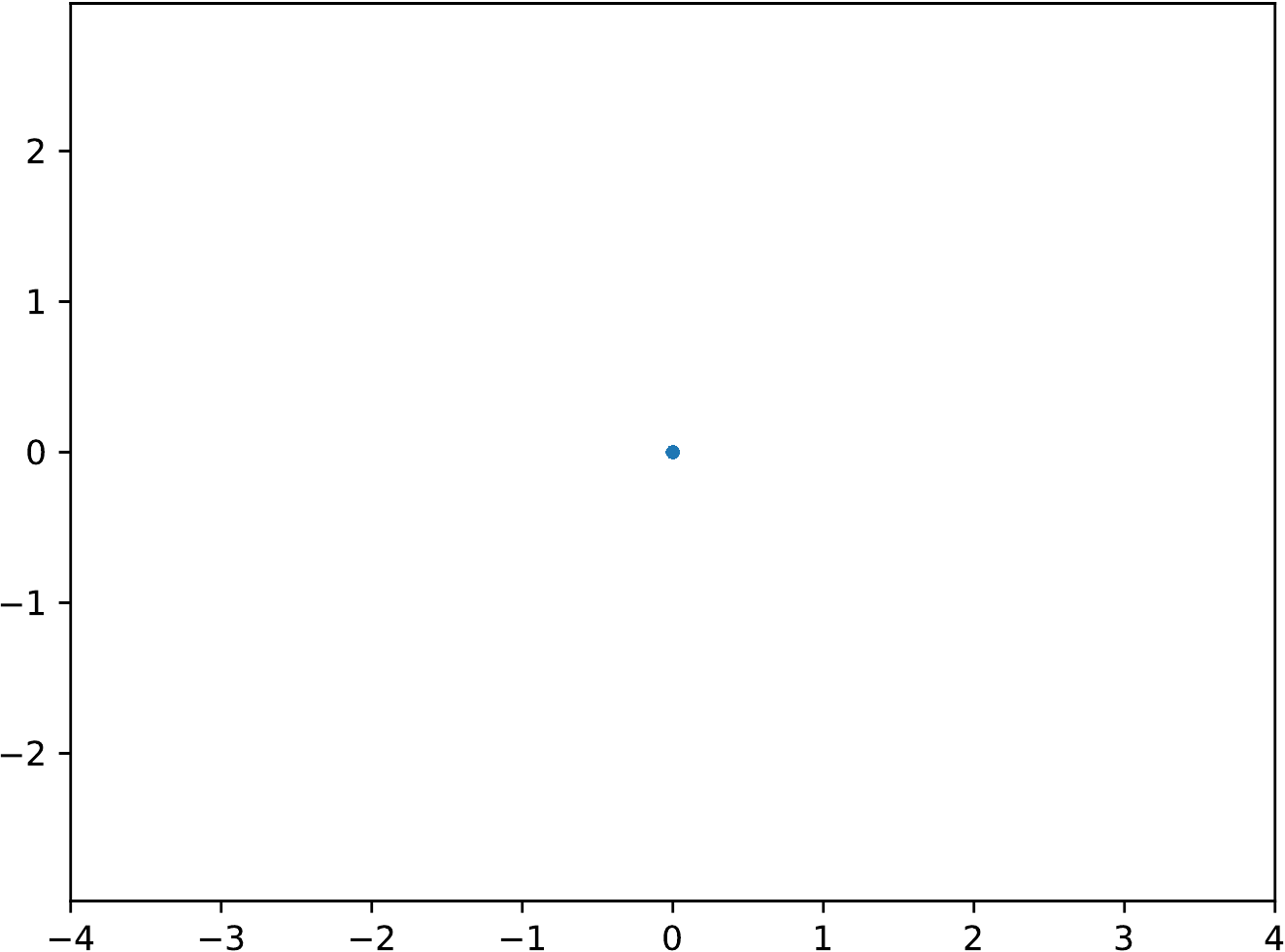} \\
	$q_0= $MLP & $q_0= $RNN
\end{tabular}
\caption{ADE with different samplers. \label{fig:ade_sampler}}
\vspace{-2mm}
\end{wrapfigure}
In \figref{fig:vis_q} in main paper we have studied \modelshort{} with different design choices of $q_0$, where a weak $q_0$ like fully factored distribution can still get reasonable results. Instead in \figref{fig:ade_sampler} we can see that, for ADE, different parameterizations of the sampler will make quite different behaviors. The MLP sampler is an autoregressive one with non-sharing parameters, while the RNN sampler has the shared parameters across different steps. This clearly shows the limitation of autoregressive model with parameter sharing, and also the necessity of learning sampler with local search to improve the weak initial sampler $q_0$.

\paragraph{Implementation details}

Here we provide more details on the instantiation of \modelshort{} on the synthetic tasks. Below we first cover the parameterization details. 

The energy function is an MLP with dimensions of $[32, 256, 256, 256, 1]$, where 32 is the input size, and $256$ is the hidden layer size. We use ELU as the activation function. 

For ADE and \modelshort{}, the $q0$ is parameterized with either autoregressive model or a factorized model. For the factorized model, we simply learn 32-dimensional vector that represents the logits of each dimension independently. For the autoregressive model, there can be two choices. The first one uses LSTM (which we denote as \texttt{RnnSampler}) to encode the bits, where LSTM has hidden size of 256 and 1 layer. All the dimensions share the same predictor that predicts the binary bit from the latent embedding obtained by LSTM. The predictor is an MLP with size $[256, 512, 2]$ with ELU activation. Another alternative is to use MLP to encode the bits, as we know the maximum length is 32 beforehand (which is not practical in general). This way we encode the history using 31 MLPs, where the $i$-th MLP has size $[i, 512, 512, 256]$ that embeds the prefix of length $i$, and use the shared predictor to predict the bit at current position. 

\modelshort{} has additional components, which are editor $q_A(\cdot|\cdot)$ and stop policy $q_{\text{stop}}$. The editor only needs to predict the location for modification, as once the location is given one can simply flip that bit. It is parameterized into $[32, 512, 512, 32]$ with ELU as activation function and softmax at the end. The stop policy is parameterized by an MLP with layers $[32, 512, 512, 1]$ with ELU activation and sigmoid in the last output. 

We use the \texttt{Inverse proposal} where $A'(\cdot|\cdot)$ is a uniform distribution that samples a random location for modification. To avoid sampling the same position twice, we first permute the locations and then pick the first $k$ locations as the proposal trajectory, where $k$ is the number of edits that is sampled from a geometric distribution, with the truncation at 16.

\subsection{Program synthesis experiments}
\label{app:exp_robustfill}

\noindent\textbf{Grammar:} We use the following grammar for RobustFill programs.

{
\small
\grammarindent24ex
\grammarparsep0.5ex
\begin{grammar}
<program> $\rightarrow$ <ExprList>

<ExprList> $\rightarrow$ <expr> | <expr> <ExprList>

<expr> $\rightarrow$ `ConstStr' <ConstExpr> | `SubStr' <SubstrExpr>

<ConstExpr> $\rightarrow$ `]' | `,' | `-' | `.' | `@' | `'' | `"' | `(' | `)' | `:' | `\%'

<SubstrExpr> $\rightarrow$ <Pos> <Pos>

<Pos> $\rightarrow$ <ConstPos> | <RegPos>

<ConstPos> $\rightarrow$ -4 | -3 | -2 | -1 | 0 | 1 | 2 | 3 | 4

<RegPos> $\rightarrow$ <ConstTok> | <RegexTok>

<ConstTok> $\rightarrow$ <ConstExpr> <p2> <direct>

<RegexTok> $\rightarrow$ <RegexStr> <p2> <direct>

<p2> $\rightarrow$ <ConstPos>

<direct> $\rightarrow$ `Start' | `End'

<RegexStr> $\rightarrow$ `[A-Z]([a-z])+' | `[A-Z]+' | `[a-z]+' | `\\d+' | `[a-zA-Z]+' | `[a-zA-Z0-9]+' | `\\s+' | `^' | `\$'

\end{grammar}
}

\noindent\textbf{Data generator:} We use following configurations for generating synthetic data for program synthesis:
\begin{itemize}
	\item The maximum number of types of tokens in input strings is set to 5. 
	\item The maximum length of input strings is 20.
	\item The maximum length of output strings is 50.
	\item The total number of input-output examples per synthesis task is 10.
	\item The number of public input-output example pairs is 4.
	\item The number of private input-output example pairs is 6.
\end{itemize}

The learned synthesizer uses the 4 public IO pairs for synthesize the program, and evaluate against all 10 IO pairs. It is considered correct if it is consistent with these 10 IO pairs. 

\noindent\textbf{Parameterization:} We use a 3-layer LSTM with hidden size of 256 to encode each input and output sequences, respectively. Then each IO pair is represented by concatenating the sequence embeddings of input and output strings. The set of inputs is obtained by max-pooling over the IO-pair embeddings, which will be served as the context for program synthesis. 

For $q_0$ we use a 3-layer LSTM with hidden size of 256 for predicting program tokens. For~\modelshort{} we parameterize the $q_A$ with two components: the position predictor $q_{\M{pos}}$ and the modified expression $q_{\M{expr}}$. $q_{\M{pos}}$ embeds the current program using 3-layer bidirectional LSTM, and predict the position using pointer mechanism~\citep{vinyals2015pointer}. Note that the selected position must be the start or end of an existing $<expr>$ in above grammar, which indicates whether we want to modify or insert a new $<expr>$ in this position. $q_{\M{expr}}$ predicts the new expression using another 3-layer LSTM, and is allowed to make empty prediction (which corresponds to delete an expression in current program). As the program heavily relies on the context free grammar to make it valid, we utilize the technique in grammarVAE~\citep{kusner2017grammar} to mask out invalid production rules during program generation.

\subsection{Fuzzing experiment}

\begin{table*}[h]
\centering
\begin{tabular}{cccc}
	\toprule
	Software & \# seed files & file size (bytes) & \# training samples for \modelshort{} \\
	\hline
	libpng & 170 & 104 - 12,901 & 146,507\\
	openjpeg & 36 & 233 - 7,885,684 & 27,572,688\\
	libmpeg2 & 131 & 10,581 - 50,000 & 6,119,237 \\
	\bottomrule
\end{tabular}
\caption{Data statistics for generative fuzzing experiments. We use window size 64 for \modelshort{} to obtain chunks of data from the raw byte streams. \label{tab:fuzz_data}}
\end{table*}

\noindent\textbf{Data statistics:}  We test different approaches against three target softwares. The OSS-Fuzz project comes with different set of seed inputs for different target softwares. These inputs are served as training samples for both \modelshort{} and \citet{godefroid2017learn}, and will be used as seed inputs for \texttt{libFuzzer} as well. \tabref{tab:fuzz_data} displays the data statistics. Note that \modelshort{} trains a conditional EBM with chunked data from the original raw byte streams, in order to handle huge files. We use chunk size 64 by default. Thus for a file with size $L$ where $L \geq 64$, there will be $L - 64 + 1$ training samples for \modelshort{}. 

\noindent\textbf{Parameterization:} We use a three-layer MLP to parameterize the energy function, where for the input layer, we use embedding size equals to 4 for the byte string. For the negative sampler, we parameterize $q_0$ with LSTM.  $q_A$ consists of two parts, namely $q_{\M{pos}}$ which predicts which position to modify using an MLP, and $q_{\M{value}}$ which predicts a new value for that position using another MLP. We use Eq~\eqref{eq:inv_proposal} for training such EBM. 

\end{appendix}

\end{document}